\documentclass[11pt,english]{article}
\usepackage{babel,float}
\usepackage[small,bf]{caption}
\usepackage[dvipsnames]{xcolor}
\RequirePackage[colorlinks,citecolor=blue,urlcolor=blue]{hyperref}

\usepackage[left=1in,right=1in,top=1.in,bottom=1.in]{geometry}
\usepackage{amsmath, amsthm, amssymb,multirow}
\RequirePackage[numbers]{natbib}
\usepackage{setspace} 
\usepackage{epsfig,psfrag,url,verbatim}
\usepackage{graphics}
\usepackage{graphicx}
\usepackage{subfig}
\usepackage{titlesec}

\titlespacing*{\paragraph}
{0pt}{1.ex plus .2ex minus .2ex}{1.ex plus .1ex}
\usepackage{booktabs}       
\usepackage{wrapfig}
\usepackage{amsmath,color}

\newtheorem{thm}{Theorem} 
\newtheorem{lemma}{Lemma}

\newtheorem{remark}{Remark}

\newcommand{\R}{\mathbb{R}}
\newcommand{\T}{\top}

\newcommand{\lam}{\lambda}

\newcommand{\ep}{\epsilon}

\newcommand{\lt}{\left}
\newcommand{\rt}{\right}

\newcommand{\td}{\tilde}

\newcommand{\cB}{\mathcal{B}}

\newcommand{\cD}{\mathcal{D}}

\newcommand{\cI}{\mathcal{I}}
\newcommand{\cJ}{\mathcal{J}}

\newcommand{\cP}{\mathcal{P}}

\DeclareMathOperator*{\argmin}{arg\,min}

\DeclareMathOperator*{\Conv}{Conv}

\newcommand*\centermathcell[1]{\omit\hfil$\displaystyle#1$\hfil\ignorespaces}

\newcommand{\B}{\boldsymbol}
\newcommand{\M}{\mathbf}

\newcommand{\sign}{\text{sign}}
\newcommand{\sbt}{\mathrm{s.t. }}
\newenvironment{myarray}[2][1]
{\def\arraystretch{#1}\array{#2}}
{\endarray}

\begin{document}
\title{Solving L1-regularized SVMs and related linear programs: Revisiting the effectiveness of Column and Constraint Generation}
\author{Antoine Dedieu\thanks{[\texttt{adedieu@mit.edu}]. Operations Research Center, MIT. Now at Vicarious AI.}~~~~~Rahul Mazumder\thanks{[\texttt{rahulmaz@mit.edu}]. MIT Sloan School of Management and Operations Research Center, MIT}~~~~Haoyue Wang\thanks{[\texttt{haoyuew@mit.edu}]. Operations Research Center, MIT} }

\date{May, 2021}

\maketitle
\begin{abstract}
The linear Support Vector Machine (SVM) is a classic classification technique in machine learning. 
Motivated by applications in modern high dimensional statistics, we consider penalized SVM problems involving the minimization of a hinge-loss function with a convex sparsity-inducing regularizer such as: the L1-norm on the coefficients, its grouped generalization and the sorted L1-penalty (aka Slope).
	Each problem can be expressed as a Linear Program (LP) and is computationally challenging when the number of features and/or samples is large -- the current state of algorithms for these problems is rather nascent when compared to the usual L2-regularized linear SVM. To this end, we propose new computational algorithms for these LPs
	by bringing together techniques from (a) classical column (and constraint) generation methods and (b) first order methods for non-smooth convex optimization --- techniques that are rarely used together for solving large scale LPs.
	These components have their respective strengths; and while they are found to be useful as separate entities, they have not been used together in the context of solving large scale LPs such as the ones studied herein. 
	Our approach complements the strengths of (a) and (b) --- leading to a scheme that seems to significantly outperform commercial solvers as well as
	specialized implementations for these problems. We present numerical results on a series of real and synthetic datasets demonstrating the surprising 
	effectiveness of classic column/constraint generation methods in the context of challenging LP-based machine learning tasks.
\end{abstract}

\setstretch{1.2}

\section{Introduction}\label{sec: introduction}

The linear Support Vector Machine (SVM)~\cite{vapnik, FHT-09-new} is a fundamental tool for binary classification. Given training data $(\mathbf{x}_i,y_i)_{i=1}^n$ with feature vector $\M{x}_{i} \in \mathbb{R}^p$ and 
label $y_{i} \in \{-1, 1\}$,
the task is to learn a linear classifier of the form $\sign( \mathbf{x} ^T \B{\beta} + \beta_0 )$ where, $\beta_0 \in\mathbb{R}$ is the offset. 
 The popular L2-regularized linear SVM (L2-SVM) considers the minimization problem
\begin{equation} \label{l2-SVM}
\min \limits_{ \B{\beta} \in \mathbb{R}^p , \beta_0 \in \mathbb{R} } 
\sum \limits_{i=1}^n \left( 1 - y_i( \mathbf{x}_i^T \B{\beta} + \beta_0  ) \right)_+  + \frac{\lambda}{2} \| \B{\beta} \|_2^2
\end{equation}
where, $(a)_+:=\max \{ a, 0 \}$ is often noted as the hinge-loss function;
and $\lambda\geq0$ regularizes the L2-norm of the coefficients $\B\beta$.
Several algorithms have been proposed to efficiently solve Problem~\eqref{l2-SVM}.
Popular approaches include stochastic subgradient methods on the primal form~\cite{bottou,pegasos},
coordinate descent methods on a dual~\cite{dual_results} and cutting plane algorithms~\cite{linear-training,optimized-cutting-planes}. 

\paragraph{The L1-SVM estimator:} The L2-SVM estimator generally leads to a dense estimate for $\B\beta$---towards this end, the L1 penalty~\cite{l1-svm-definition,FHT-09-new} is often used as a convex surrogate to encourage 
sparsity (i.e., few nonzeros) in the coefficients. This leads to one of the problems we consider in this paper, namely, the L1-SVM problem:
\begin{equation} \label{l1-SVM}
\min \limits_{ \B{\beta} \in \mathbb{R}^p , \beta_0 \in \mathbb{R} } 
\sum \limits_{i=1}^n \left( 1 - y_i( \mathbf{x}_i^T \B{\beta} + \beta_0  ) \right)_+  + \lambda \| \B{\beta} \|_1, 
\end{equation}
which can be written as a Linear Program (LP), as shown in Section~\ref{sec:pdual-svm}.
The regularization parameter $\lambda \geq 0$ controls the L1-norm of $\B\beta$. Off-the-shelf solvers, including commercial LP solvers (eg, Gurobi, Cplex) work very  well for small/moderate sized problems, but become expensive in solving Problem~\eqref{l1-SVM} when $n$ and/or $p$ is large (e.g. when $n\approx p \approx 10^4$; or $p\approx 10^6$ and $n$ is a few hundreds). 
{ Some high-quality specialized solvers for Problem~\eqref{l1-SVM} include: a homotopy based method to compute the entire (piecewise linear) regularization path in $\B\beta$~\cite{path-SVM}; 
 methods based on Alternating Direction Method of Multipliers (ADMM) or operator splitting~\cite{admm,scs}.
\cite{pang2017parametric} proposes a parametric simplex (PSM) approach to solve Problem~\eqref{l1-SVM}---this leads to a pair of primal/dual solutions at optimality, and the authors demonstrate that their method achieves state-of-the-art performance for some LP-based sparse learning tasks\footnote{The largest example studied in this work is the Dantzig Selector problem with $n=200$ samples and $p=5,000$ features.} compared to a benchmark ADMM-based implementation ``flare"~\cite{li2015flare}.
Our experiments suggest that for problems with small $n \approx 100$ and large $p \approx 50,000$, PSM works well. However, this becomes inefficient as soon as $n$ is large (e.g., for $n \approx 10^4$ and $p\approx 100$, PSM can take hours and large memory, while the methods we propose here take a few seconds with minimal memory). 
 	\cite{mangasarian2006exact} propose a perturbation approach where they reformulate the L1-SVM problem as an unconstrained smooth minimization problem by including an additional regularization term. They apply Newton type methods to solve the resulting problem. Such Newton-type methods as discussed in~\cite{mangasarian2006exact} may not scale to large problem instances, due to expensive matrix inversions.}
In this paper, our goal is to propose new computational algorithms for the L1-SVM LP by revisiting
classical operations research tools such as column and constraint generation with origins in 1950s~\cite{ford1958suggested}---these methods appear to have been somewhat underutilized in the context of the L1-SVM problem (as we discuss below) and relatives of the L1-penalty, that we consider here. To improve the performance of column/constraint generation-based methods we use
(relatively newer) first order optimization techniques.

{We note that there are several appealing L1-regularized classifiers and efficient algorithms that consider a smooth loss function (e.g, logistic, squared hinge loss, etc)---see~
\cite{friedman2010regularization,yuan2010comparison}
(for example). Different loss functions have different operating characteristics: in particular, smooth loss functions lead to estimators that are different from the hinge loss as in the L1-SVM problem~\cite{FHT-09-new}. Our goal in this paper is not to pursue an empirical analysis of the relative merits/de-merits of different loss functions (which are well-known); rather, we focus on algorithms for the hinge loss function, with additional penalty functions that are representable as linear programs.} 

\paragraph{The Group-SVM estimator:} In several applications, sparsity is structured --- the coefficient indices are naturally found to occur in groups that are known a-priori and it is desirable to select (or set to zero) a whole group together as a ``unit''. 
In this context, a group version of the usual L1 norm is often used to improve the performance and interpretability of the model~\citep{yuan2006model, huang2010benefit}. We consider the popular 
L1/$L_\infty$ penalty~\cite{bach2011convex} leading to the  
\emph{Group-SVM} Problem:
\begin{equation} \label{group-SVM-intro}
\min \limits_{ \B{\beta} \in \mathbb{R}^p , \ \beta_0 \in \mathbb{R} } 
\sum \limits_{i=1}^n \left( 1 - y_i( \mathbf{x}_i^T \B{\beta} + \beta_0  ) \right)_+  + \lambda \sum_{g=1}^G \| \B{\beta}_g \|_{\infty}
\end{equation}
where, $g=1, \ldots, G$ denotes a group index (the groups are disjoint), $\B\beta_{g}$ denotes the subvector of coefficients belonging to group $g$ and
$\B{\beta} = \left(\B{\beta}_1, \ldots, \B{\beta}_G \right)$. Problem \eqref{group-SVM-intro} can be expressed as an LP and our approach applies to this problem as well (with suitable modifications).

\paragraph{The Slope-SVM estimator:} The third  problem we study in this paper is of a different flavor and is inspired by the sorted L1-penalty aka the Slope norm~\citep{slope-proximal,bellec2018slope}, popularly used in the context of penalized least squares problems for its useful statistical properties.  
For a (a-priori specified) sequence $\lambda_1 \geq \ldots \geq \lambda_p \geq 0$,  the Slope-SVM problem is given by:
\begin{equation} \label{Slope-intro}
\min \limits_{ \B{\beta} \in \mathbb{R}^p , \ \beta_0 \in \mathbb{R} } 
\sum \limits_{i=1}^n \left( 1 - y_i( \mathbf{x}_i^T \B{\beta} + \beta_0  ) \right)_+  + \sum_{j=1}^p \lambda_j | \beta_{(j)} |,
\end{equation}
where $|\beta_{(1)}| \ge \ldots \ge |\beta_{(p)}|$ are the ordered values of $|\beta_{i}|, i =1, \ldots, p$. Unlike Problems~\eqref{l1-SVM} and~\eqref{group-SVM-intro} where the penalty function is separable (respectively, block separable), the penalty function in~\eqref{Slope-intro} is not separable in the coefficients.
We show in Section \ref{sec: group-L1-SVM} that Problem \eqref{Slope-intro} can be expressed as an LP with 
$O(n+p)$ variables and an exponential number (in $p$) of constraints, consequently posing challenges in optimization. Despite the large number of constraints, we show that the LP can be solved using a non-standard application of column/constraint generation. 
We note that using standard reformulation methods~\cite{boyd2004convex} (see Section~\ref{epi-slope-1}), Problem~\eqref{Slope-intro} 
can be modeled (e.g., using {\texttt{CVXPY}}) and solved (e.g., using a commercial solver like Gurobi) for small-sized problems. However, the computations become expensive when 
$\lambda_{i}$s are distinct (e.g., as in~\cite{slope-proximal}) --- for these cases,  
{\texttt{CVXPY}} can handle problems up to $n \approx 100$, $p\approx 200$ whereas, our approach can solve problems with $p \approx 50,000$ within a few seconds.

\paragraph{First order methods:} First order methods~\cite{nesterov2004introductorynew} have enjoyed great success in solving large scale structured convex optimization problems arising in machine learning applications. Many of these methods (such as proximal gradient and its accelerated variants) are appealing candidates for  the 
minimization of smooth functions and also problems of the composite form~\cite{nesterov2013gradient}, wherein accelerated gradient methods enjoy a convergence rate of $O(1/\sqrt{\epsilon})$ to obtain an $\epsilon$-accurate solution. 
For the nonsmooth SVM problems (\eqref{l1-SVM},\eqref{group-SVM-intro},\eqref{Slope-intro}) discussed above, Nesterov's smoothing method~\cite{nesterov-smoothing} (which replaces the hinge-loss with a smooth approximation)
 can be used to obtain algorithms with a convergence rate of 
$O(1/{\epsilon})$---a method that we explore in Section~\ref{sec: FOM}. 
While this procedure (with additional heuristics based on~\cite{tibshirani2012strong}) lead to low accuracy solutions relatively fast; in our experience, the basic version of this algorithm 
takes a long time to obtain a solution with
 higher accuracy when $n$ and/or $p$ are large.
Similarly, first order methods based on~\cite{becker2011templates} and~\cite{scs,admm} also experience increased run times as the problem sizes become large.

\smallskip 

\noindent {\bf{What this paper is about:}}  
In this paper, we propose an efficient algorithmic framework for L1-SVM, Group-SVM and Slope-SVM using tools in column/constraint generation (and, leveraging the capabilities of excellent LP solvers) that make use of
some basic structural properties of solutions to these problems, as we discuss below. 

Specifically, large values of $\lambda$ will encourage an optimal solution to Problem~\eqref{l1-SVM}, $\hat{\B\beta}$ (say), 
 to be sparse. This sparsity will be critical to solve Problem~\eqref{l1-SVM} when $p \gg n$---we anticipate to solve Problem~\eqref{l1-SVM} without having to create an LP model with all $p$ variables. To this end, we use column generation, a classical method in mathematical optimization/operations research originating in the context of solving integer programs during late 1950s~\cite{ford1958suggested,dantzig1960decomposition} (see also~\cite{desrosiers2005primer} for a nice review). 
{We also make use of another structural aspect of a solution to Problem~\eqref{l1-SVM} when
 $n$ is large (and $p$ is small). Suppose most of the samples can be classified correctly via a linear classifier---then, at an optimal solution, $1 \leq y_i( \mathbf{x}_i^T \B{\beta} + \beta_0)$ and hence $\tilde{\alpha}_{i}:=(1 - y_i( \mathbf{x}_i^T \B{\beta} + \beta_0))_+$ will be zero for many indices $i=1,\ldots, n$.}
We leverage this sparsity in $\tilde{\alpha}_{i}$'s to develop efficient algorithms for Problem~\eqref{l1-SVM}, 
using constraint generation~\cite{ford1958suggested,desrosiers2005primer} methods.
This allows us to solve~\eqref{l1-SVM} without explicitly creating an LP model with $n$ samples.

To summarize, there are two characteristics special to an optimal solution of Problem~\eqref{l1-SVM}:
{(a)}~sparsity in the SVM coefficients, i.e., $\B\beta$ and/or 
{(b)}~sparsity in $\tilde{\alpha}_{i}$'s.  Column generation can be used to handle (a); constraint generation can be used to address (b) --- in problems where both $n,p$ are large, we propose to combine both 
column and constraint generation. To our knowledge, while column generation and constraint generation are used separately in the context of solving large scale LPs, using them together, 
in the context of the L1-SVM problem is novel. 
For solving these (usually small) subproblems, we rely on powerful LP solvers (e.g., simplex based algorithms of Gurobi) which lead to a pair of primal-dual solutions and also possess excellent warm-starting capabilities.
Our approach applies to the Group-SVM Problem~\eqref{group-SVM-intro} with suitable modifications.
We also extend our approach to handle the Slope-SVM problem~\eqref{Slope-intro}, which requires a fairly involved use of column/constraint generation. 
Numerical evidence presented here suggests that column/constraint generation methods are modular, simple and powerful tools---they should perhaps be considered more frequently to solve machine learning tasks based on LPs, even beyond the ones studied here.

The column/constraint generation methods mentioned above, are found to benefit from a good initialization. 
To this end, we use  first order optimization methods to get approximate solutions with low computational cost. These solutions serve as decent initializations and are subsequently improved to deliver optimal solutions as a part of our column and/or constraint generation framework. This approach is found to be useful in all the three problems studied here.

To our knowledge, this is the first paper that brings together first order methods in  convex optimization and column/constraint generation algorithms for solving large scale LPs, in the context of solving a problem of key importance in machine learning. Implementation of our methods can be found at:\\
{\small \url{https://github.com/wanghaoyue123/Column-and-constraint-generation-for-L1-SVM-and-cousins}}.

\noindent {\bf Organization of paper:} The rest of this paper is organized as follows. Section \ref{sec: CG} presents an overview of 
column/constraint generation methods; and then discusses their instantiation 
for the L1-SVM and Group-SVM problems. 
Section \ref{sec: group-L1-SVM} discusses the Slope-SVM problem.
Section \ref{sec: FOM} discusses how first order methods can be used to get approximate solutions for these problems.
Section \ref{sec: computation} presents numerical results.

\smallskip

\noindent {\bf Notation:}
For an integer $a$ we use $[a]$ to denote $\{1, 2, \ldots, a\}$. The $i$th entry of a vector $\M{u}$ is denoted by $u_i$.
For a set ${\mathcal A}$, we use the notation $| \mathcal A |$ to denote its size. For a positive semidefinite matrix $\M{A}$, 
we denote its largest eigenvalue by $\sigma_{\max}(\M{A})$. For a vector $\M{x}\in \R^n$ and a subset $\cB \subseteq [n]$, we let $\M{x}_{\cB}$ denote the sub-vector of $\M{x}$ corresponding to the indices in $\cB$.

\section{Column and constraint generation for L1-SVM and its group extension}\label{sec: CG}

{
As we allude to earlier, column and constraint generation algorithms have a long history in 
mathematical optimization and operations research~\cite{ford1958suggested,dantzig1960decomposition}.
Here we present an outline  of these methods for a generic LP. We subsequently discuss their applications to Problems~\eqref{l1-SVM} and~\eqref{group-SVM-intro}.
}

\subsection{Methodology for column and constraint generation}\label{sec: CG-LP}

{
The basic idea of column generation is to start with a candidate set of columns and incrementally add new columns into the model until some optimality conditions are met.
Consider the \textit{primal} LP problem \textbf{(P)} where, $\bar{n}, \bar{p}$ are integers and $\mathbf{A} \in \mathbb{R}^{\bar{n} \times \bar{p}}, \ \mathbf{b} \in \mathbb{R}^{\bar{n}}, \ \mathbf{c} \in \mathbb{R}^{\bar{p}}$ are problem data.
We  assume that the optimal objective value of \textbf{(P)}  is finite.
The strong duality theorem (cf \cite{LP}, Theorem 4.4) states that this optimum is equal to that of the \textit{dual} problem \textbf{(D)}: 
}
\begin{equation*}\label{classical-lp}
\begin{myarray}[1.1]{c c c c c c c c c r}
\mathbf{(P)}: & \min \limits_{\B{\theta} \in \mathbb{R}^{\bar{p}}} &  \mathbf{c}^T \B{\theta}  & &  ~~~~~~~~~~~~~~~~~~~~
& \mathbf{(D)}: & \max \limits_{ \mathbf{q} \in \mathbb{R}^{\bar{n}}}  &  \mathbf{q}^T \mathbf{b}  & &\\
& \sbt & \mathbf{A}\B{\theta} \ge \mathbf{b}, & \B{\theta} \ge \M{0}   & ~~~~~~~~~~~~~~~~~~~~
& & \sbt & \mathbf{q}^T \mathbf{A}  \le \mathbf{c} & \mathbf{q} \ge \M{0}. &\\
\end{myarray}
\end{equation*}
Let us consider the case where $\bar{p}$ is large compared to $\bar{n}$ and we anticipate an optimal solution of $(\M{P})$ to have few nonzeros. 
Let $\M{A}_{j}$ denote the $j$th column of $\M{A}$. Consider a subset of columns 
$\mathcal{B} \subset [\bar{p}]$ and the corresponding \emph{reduced} primal/dual problems:
\begin{equation*}\label{classical-lp-restricted}
\begin{myarray}[1.1]{c c c c c c  c }
(\mathbf{P}_{\mathcal B}): & \min\limits_{\B\theta_{\mathcal{B}} \in \mathbb{R}^{|\mathcal{B}|}} &  \sum\limits_{j \in \mathcal{B}} {c}_{j}\theta_{j}  
&~~~~& (\mathbf{D}_{\mathcal B}): & \max \limits_{ \mathbf{q} \in \mathbb{R}^{\bar{n}}}  &  \mathbf{q}^T \mathbf{b}  \\
& \sbt & \sum\limits_{j \in {\mathcal B}}\mathbf{A}_{j} {\theta}_{j} \ge \mathbf{b},~~\B{\theta}_{{\mathcal B}} \ge \M{0}  &~~~&& \sbt & \M{q}^T \M{A}_j  \leq {c}_{j}, ~ j \in {\mathcal B}, ~~\mathbf{q} \ge \M{0}. 
\end{myarray}
\end{equation*}

Let $\tilde{\B\theta}_{\mathcal B}$ and $\tilde{\M{q}}$ be a pair of primal/dual solutions to the restricted problems $(\mathbf{P}_{\mathcal B})$ and $(\mathbf{D}_{\mathcal B})$. 
Let $\hat{\B\theta}$ be an extension of $\tilde{\B\theta}_{\mathcal B}$ to $\mathbb{R}^{\bar{p}}$ i.e., 
$\hat{\B\theta}_{\mathcal B}=\tilde{\B\theta}_{\mathcal B}$ and $\hat{\B\theta}_{{\mathcal B}^c}=\M{0}$.
While $\hat{\B\theta}$ is a feasible solution for $(\M{P})$, it may not be optimal for $(\M{P})$. An optimality certificate can be obtained via $\tilde{\M{q}}$, by checking if $\tilde{\M{q}}$ is feasible for $(\M{D})$. 
Specifically, let the \emph{reduced cost} for variable $j$ be defined as $\bar{c}_j:=c_{j} - \tilde{\M{q}}^T\M{A}_j$. 
If $\bar{c}_{j} \geq 0$ for all $j \notin {\mathcal B}$, then $\tilde{\M{q}}$ is optimal for $(\M{D})$; and $\hat{\B\theta}$ is an optimal solution for $(\M{P})$.\\
\noindent {\bf Column generation:} Column generation makes use of the scheme outlined above. It is useful when $\bar{p}\gg \bar{n}$, and an optimal solution to $(\M{P})$ has few nonzeros. We start with a subset of columns say, ${\mathcal B}$---i.e., a guess for the support of a minimizer of $(\M{P})$ for which $(\mathbf{P}_{\mathcal B})$ is feasible. Given $\mathcal B$, we solve the \emph{restricted} problem $(\M{P}_{\mathcal B})$; and have a pair of primal/dual solutions for $(\M{P}_{\mathcal B})$/$(\M{D}_{\mathcal B})$.
If all the reduced costs are nonnegative, we declare convergence and stop. Otherwise, we find a column (or a subset of columns) outside $\mathcal B$ with the most negative reduced cost(s), update $\mathcal B$, and re-solve the updated problem $(\M{P}_{\mathcal B})$ by making use of the warm-start capabilities of a simplex-based 
LP solver. If ${\mathcal B}$ is only allowed to increase (i.e., we do not drop variables) then this process converges after finitely many iterations.
Convergence guarantees of this procedure is formally discussed in Section~\ref{sec:conv-CG}.
Upon termination, column generation leads to a pair of primal/dual optimal solutions to $(\M{P})$/$(\M{D})$.\\
\noindent{\bf Constraint generation:} We now consider the case when $\bar{n} \gg \bar{p}$. Suppose at an optimal solution to \textbf{(P)}, only a small fraction of the $\bar{n}$ constraints $\mathbf{a}_i^T\B\theta \geq b_{i}$ for $i \in [\bar{n}]$ are active or binding. Then, an optimal solution can be obtained by considering only a small subset of the $\bar{n}$ constraints. 
This inspires the use of a constraint generation algorithm, which can also be interpreted as column generation~\cite{LP} on the dual Problem \textbf{(D)}.

\subsection{Primal and dual formulations of L1-SVM}\label{sec:pdual-svm}
We present an LP formulation for Problem~\eqref{l1-SVM}:

	\begin{subequations}\label{l1-SVM-primal-main} 
		\begin{align}
		\boxed{\mathcal{P}_{ \lam} \left( [n], [p] \right)}~~~~~~~~~~~~&\min \limits_{ \substack{ \B{\xi} \in \mathbb{R}^n, \beta_0 \in \mathbb{R}  \\  \B{\beta}^+,\  \B{\beta}^- \in \mathbb{R}^p } } &  \sum \limits_{i=1}^n \xi_i  + \lambda \sum \limits_{j=1}^p \beta^+_j + \lambda \sum \limits_{j=1}^p \beta^-_j& \label{l1-SVM-primal}   \\
		&\sbt & \;\; \ \xi_i + y_i \mathbf{x}_i^T \B{\beta}^+ - y_i \mathbf{x}_i^T \B{\beta}^-  + y_i \beta_0 \ge 1 &  \;\;\;\;\; i \in [n]  \label{line-1-l1svm-const} \\ 
		&& \;\;\;\; \ \B{\xi} \ge 0, \ \B{\beta}^+ \ge 0, \ \B{\beta}^- \ge 0.  &  \nonumber
		\end{align}
	\end{subequations}

Above, the positive and negative parts of $\beta_i$ are denoted as $\beta^+_i = \max \{\beta_i, 0\}$ and $\beta^-_i = \max \{-\beta_i, 0 \}$ respectively,
and $\xi_{i}$'s are auxiliary continuous variables corresponding to the hinge-loss function. 
The feasible set of Problem \eqref{l1-SVM-primal-main} is nonempty. A dual of~\eqref{l1-SVM-primal-main} is the following LP:

	\begin{equation}\label{l1-SVM-dual}
	\begin{myarray}[1.1]{c c c c r}
	\boxed{\mathcal{D}_{ \lam} \left( [n], [p] \right)}~~~~~~~~\;\;\;\;& \max \limits_{\B{\pi} \in \mathbb{R}^{n } } & \centermathcell{ \sum \limits_{i=1}^n \pi_i }   \\
	&\;\;\; \sbt  & \centermathcell{ \;\;\; -\lambda \le  \sum \limits_{i=1}^n y_i x_{ij} \pi_i \le \lambda \;\;\; }   & \;\;\;\;\; j \in [p] \\ 
	& & \centermathcell{  \mathbf{y}^T \B{\pi} = 0 }  \\
	& &  \centermathcell{ 0\le \pi_i \le 1 } & \;\;\;\;\; i \in [n]. 
	\end{myarray}
	\end{equation}
For Problems~\eqref{l1-SVM-primal-main} and \eqref{l1-SVM-dual}, standard complementary slackness conditions lead to:
\begin{equation}\label{Kuhn-Tucker}
\begin{myarray}[1.1]{c c c c r}
(1-\pi_i) \xi_i = 0, &  & \pi_i \left( \xi_i + y_i \mathbf{x}_i^T \B{\beta} + y_i \beta_0 - 1  \right)  = 0  & \;\;\;\;\; i \in [n].
\end{myarray}
\end{equation}
Let $(\B{\beta}^*(\lambda), \beta_0^*(\lambda))$ and $\B{\pi}^*(\lambda)$ denote optimal solutions for Problems \eqref{l1-SVM-primal-main} and \eqref{l1-SVM-dual}. In what follows, for notational convenience, we will drop the dependence (of an optimal solution) on $\lambda$ when there is no confusion.  We make a few observations regarding the geometry of an L1-SVM solution following standard SVM terminology~\cite{FHT-09-new}. For easier notation, 
we denote $\alpha_{i}=y_i \mathbf{x}_i^T (\B{\beta}^+ - \B\beta^{-})+ y_i \beta_0$ for all $i$. 
Note that $\xi_{i} = \max \{ 1-\alpha_{i}, 0\}$ for all $i$. 
{If a point $i$ is correctly classified, we have $\xi_{i}=0$; and if this point is away from the margin, then we have
$0> 1 - \alpha_{i}$ and hence $\pi_{i}=0$ (from~\eqref{Kuhn-Tucker}). 
Note that if point $i$ is misclassified, then $\xi_{i} >0$ and  $\pi_{i}=1$. 
Furthermore, based on the value of $\pi_{i}$, we have the following cases: 
(i) If $\pi_{i}=0$, then $\alpha_{i}\geq 1$; (ii) If $\pi_{i}=1$ then $1 \geq \alpha_{i}$ and 
(iii) If $\pi_i \in (0,1)$ then $\alpha_{i}=1$. The SVM coefficients can be estimated from the samples lying on the margin i.e., for all $i$ such that $\alpha_{i}=1$. In particular, 
if an optimal solution to the L1-SVM problem has $\kappa$-many nonzeros in $\B\beta$, then $(\B\beta,\beta_0)$ can be computed based on $(\kappa+1)$-many samples lying on the margin (assuming that the corresponding feature columns form a  full rank matrix). }

\subsection{Column generation for L1-SVM}\label{sec: SVM-CG}
We discuss how column generation applies
to the L1-SVM Problem~\eqref{l1-SVM} for a given $\lambda$---we consider cases where
$p \gg n$ (and $n$ is small).
Given a set of candidate features $\mathcal{J} \subset \left\{1,\ldots,p \right\}$ (we discuss later how to initialize $\mathcal J$), we form the \textit{restricted columns} L1-SVM problem as
	\begin{equation}\label{l1-SVM-restricted-primal}
	\begin{myarray}[1.1]{c c c r}
	\boxed{\mathcal{P}_{ \lam} \left( [n], \mathcal{J} \right)}~ & \min \limits_{ \substack{\B{\xi} \in \mathbb{R}^n, \beta_0  \in \mathbb{R} \\   \B{\beta}^+,\  \B{\beta}^- \in \mathbb{R}^{| \mathcal{J}| } } }  &  \sum \limits_{i=1}^n \xi_i  + \lambda \sum \limits_{j \in \mathcal{J} } \beta^+_j + \lambda \sum \limits_{j \in \mathcal{J}} \beta^-_j& \\
	~~~~~~&\sbt & \;\; \ \xi_i +  \sum \limits_{j \in \mathcal{J} }  y_i x_{ij} \beta_j^+ - \sum \limits_{j \in \mathcal{J} }  y_i x_{ij} \beta_j^-  + y_i \beta_0 \ge 1 & \;\;\; i\in [n]\\
	&& \;\;\;\; \ \B{\xi} \ge 0, \ \B{\beta}^+ \ge 0, \ \B{\beta}^- \ge 0
	\end{myarray}
	\end{equation}
and the corresponding dual problem is 
\begin{equation}\label{l1-SVM-colcons-dual}
\begin{myarray}[1.1]{c c c c r}
\boxed{\cD_{\lam}([n], \cJ)}:\;\;\;\;& \max \limits_{\B{\pi} \in \mathbb{R}^{n } } & \centermathcell{ \sum \limits_{i=1}^n \pi_i }   \\
&\;\;\; \sbt  & \centermathcell{ \;\;\; -\lambda \le  \sum \limits_{i=1}^n y_i x_{ij} \pi_i \le \lambda \;\;\; }   & \;\;\;\;\; j \in \cJ \\ 
& & \centermathcell{  \mathbf{y}^T \B{\pi} = 0 }  \\
& &  \centermathcell{ 0\le \pi_i \le 1 } & \;\;\;\;\; i \in [n]. 
\end{myarray}
\end{equation}

If $(\hat{\beta}^+_j,\hat{\beta}^-_j), j \in {\mathcal J}$ is a solution to~\eqref{l1-SVM-restricted-primal}, it can be extended to a feasible solution to~\eqref{l1-SVM-primal-main} by padding coordinates outside ${\mathcal J}$ 
with zeros---we let $\hat{\beta}_{j}=\hat{\beta}_{j}^+ - \hat{\beta}_j^{-}$ (for $j \in [p]$) denote the corresponding feature weights. 
Hence, the minimum of~\eqref{l1-SVM-restricted-primal} is an upper bound for that of~\eqref{l1-SVM-primal-main}.
Let $\hat{\B\pi} \in \R^n$ be an optimal solution of \eqref{l1-SVM-colcons-dual}. 
If $\hat{\B\pi}$ is feasible for~\eqref{l1-SVM-dual} then it is an optimal solution for~\eqref{l1-SVM-dual}; and 
$\hat{\B\beta}$ is an optimal solution to~\eqref{l1-SVM-primal-main}.
Otherwise, $\hat{\B\pi}$ is not an optimal solution; and some inequality constraints in~\eqref{l1-SVM-dual} are violated. 
For $\beta_{j}^+$, we denote its corresponding reduced cost (cf. Section \ref{sec: CG-LP}) by $\bar{\beta}^{+}_j$. A similar notation is used for $\beta_{j}^-$.
For every pair $\beta^+_j, \ \beta^-_j, \ j \notin \mathcal{J}$, the minimum of their reduced 
costs is
\begin{equation} \label{reduced-cost}
\min \left\{ \bar{\beta}^{+}_j,  \bar{\beta}^{-}_j \right\} = \lambda - \left|   \sum  \nolimits_{i\in[n]}  y_i x_{ij} \hat{\pi}_i  \right|.
\end{equation}
We select a subset of indices with negative reduced costs\footnote{For the column generation procedure to converge, it is not required to choose indices with the smallest reduced costs---any subset of indices with negative reduced costs can be selected.} and append it to ${\mathcal J}$ to form the new restricted L1-SVM problem; which is re-solved via LP warm-starting. The process repeats till convergence. For a tolerance level $\epsilon>0$ (e.g, $\epsilon=10^{-2}$), we update $\mathcal{J}$ by adding all columns $j$ for which the corresponding reduced cost~\eqref{reduced-cost}
is lower than `$-\epsilon$'. We summarize the algorithm below.

\medskip 

\noindent \textsc{Algorithm~1}: \textbf{Column generation for L1-SVM} \newline
\textbf{Input:} $\mathbf{X}$, $\mathbf{y}$, regularization parameter $\lambda$, a convergence threshold $\epsilon \geq 0$, a set of columns $\mathcal{J}$.
\noindent \textbf{Output: }A near-optimal solution $\hat{\B{\beta}}$ for the L1-SVM Problem \eqref{l1-SVM}.
\begin{enumerate}
	\item Repeat Steps 2 to 3 until $\mathcal{J}$ stabilizes.
	
	\item Solve the problem $\mathcal{P}_{ \lam } \left( [n], \mathcal{J} \right)$ (cf Problem~\eqref{l1-SVM-restricted-primal}).

	\item Form the set $\mathcal{J}^{\epsilon}$ of columns in $\left\{1,\ldots,p \right\} \backslash \mathcal{J}$ with reduced cost lower than $-\epsilon$.  
	Update  $\mathcal{J} \leftarrow \mathcal{J} \cup \mathcal{J}^{\epsilon}$; and go to Step 2.
\end{enumerate}

\subsubsection{Convergence of Column generation}\label{sec:conv-CG}
Algorithm~1 expands (with no deletion) the set of columns in the restricted problem---this converges in a finite number of iterations, bounded above by $p$. In the worst case, 
one may need to add all the columns in~\eqref{l1-SVM-primal-main}.
The cost of Algorithm~1 also depends upon the size of $\mathcal J$---if this becomes comparable to $p$, then the cost of solving~\eqref{l1-SVM-restricted-primal} will be large.
However, assuming that a solution to the L1-SVM problem corresponds to a sparse $\B\beta$,
then using a good initialization for $\mathcal J$ and/or regularization path continuation (discussed in Section~\ref{sec: correlation}), the worst case behavior is not observed in practice (see our results in Section~\ref{sec: computation}).
In addition, thanks to simplex warm-start capabilities, one can compute a solution to the updated version of the restricted L1-SVM problem quite efficiently.\\
An appealing aspect of the column generation framework is that it provides a bound on the optimality gap, even if one were to terminate early. We present the following result.

\begin{thm}\label{thm-one}
	Let $z^*$ and $\hat{z}$ denote the optimal objective values of the full L1-SVM Problem \eqref{l1-SVM-primal-main} and the restricted L1-SVM Problem \eqref{l1-SVM-restricted-primal}, respectively. 
	At an iteration of Algorithm~1, define
	\begin{eqnarray}
	\td\ep:= - \min_{j\in [p] \setminus \cJ} \lt\{ \min \{ \bar \beta_j^+ ,  \bar \beta_j^-  ,0  \}  \rt\} \ge 0. 
	\end{eqnarray}
	where $\bar\beta_j^+$ and $\bar\beta_j^-$ are defined in \eqref{reduced-cost}.
	If  $\B{\beta}^*$ denotes an optimal solution of Problem \eqref{l1-SVM-primal-main}, 
then it holds that: 
	$$z^* \le \hat z \le z^* + \td\ep \| \B{\beta}^* \|_{1} . $$
\end{thm}
\begin{proof}
	The first inequality $z^* \le \hat z$ is trivial (as column generation leads to a feasible solution for the full LP~\eqref{l1-SVM-primal-main}). Below we prove the second inequality. 
	From \eqref{reduced-cost} and our definition of $\td\ep$, we have
	\begin{eqnarray}
	-\lambda - \td\ep \le    \sum  \nolimits_{i\in[n]}  y_i x_{ij} \hat{\pi}_i  \le \lam  + \td\ep ~~~~ \forall j\in [p] \setminus \cJ. \nonumber
	\end{eqnarray}
	Since $\hat{\B{\pi}} $ is a  solution of $D_{\lam} ([n], \cJ)$, we also have 
		\begin{eqnarray}
	-\lambda  \le    \sum  \nolimits_{i\in[n]}  y_i x_{ij} \hat{\pi}_i  \le \lam   ~~~~ \forall j\in \cJ, \nonumber
	\end{eqnarray}
	$ \B{y}^\T {\hat{\B{\pi}}} =0 $ and $0\le\hat \pi_i \le 1$ for all $i\in [n]$---hence, $ \hat{\B{\pi}}$ is a feasible solution of $\cD_{\lam+\td\ep} ([n], [p])$. On the other hand, let $(\B{\xi}^*, \B{\beta}^{*+}, \B{\beta}^{*-}, {\beta}_0^*)$ be an optimal solution of $\cP_{\lam} ([n], [p])$. Then it is also a feasible solution of $\cP_{\lam+\td\ep} ([n], [p])$. By weak duality, we have 
	\begin{eqnarray}
	\hat z= \sum_{i=1}^n \hat \pi_i \le 
	 \sum \limits_{i=1}^n \xi_i^*  + (\lam+\td\ep) \sum \limits_{j=1}^p (\beta^{*+}_j  +  \beta^{*-}_j) = 
	 z^* + \td\ep  \| \B{\beta}^* \|_1, \nonumber
	\end{eqnarray}
	which completes the proof.
\end{proof}
Theorem~\ref{thm-one} is an adaptation of a result stated in~\cite{desrosiers2005primer} without proof. 
Theorem~\ref{thm-one} states the optimality gap is of the order of the smallest nonpositive 
reduced cost corresponding to the SVM coefficients. Note that the value $\td \ep$ satisfies $\td \ep \le \ep$ where $\ep$ is the value of tolerance we set in Algorithm~1.
When $\td\ep=0$, we are at an optimal solution.

There are variants of the column generation procedure where one can also drop variables (instead of continually expanding the set of columns in the restricted problem)---see~\cite{desrosiers2005primer}. 
This is useful if the size of $\mathcal J$ becomes so large that the restricted problem becomes difficult to solve. If we only expand the set of columns in the restricted problem, as in Algorithm~1, we obtain a sequence of decreasing objective values across the column generation iterations. If one were to both add and delete columns, one may not obtain a monotone sequence of objective values. In this case, additional care is needed to ensure convergence of the resulting procedure~\cite{desrosiers2005primer} (though Theorem~\ref{thm-one} provides an optimality gap for a given reduced cost tolerance).

\paragraph{Initializing column generation with a candidate set of columns:} 
In practice, Algorithm~1 is found to benefit from a good initial choice for $\mathcal{J}$. To obtain a reasonable estimate of $\mathcal{J}$ with low computational cost, we list a couple of options that we found to be useful.\\
\noindent (i) \emph{First order methods:} Section \ref{sec: FOM} discusses first order methods to obtain an approximate solution to the L1-SVM problem, which can be used to initialize $\mathcal{J}$. \\
(ii) \emph{Regularization path:} We compute a path (or grid) of solutions to L1-SVM (with column generation) for a decreasing sequence of $\lambda$ 
values (e.g., $\lambda \in \{ \lambda_0, \ldots, \lambda_M \}$)
with the smallest one set to the current value of interest.
This method is discussed in Section~\ref{sec: correlation}.

\subsubsection{Computing a regularization path with column generation} \label{sec: correlation}

Note that the subgradient condition of optimality for the L1-SVM Problem~\eqref{l1-SVM} is given by:
$\lambda  \sign(\beta_j^*) =  \sum \limits_{i=1}^n y_i x_{ij} \pi_i^*$
where, for a scalar $u$, $\sign(u)$ denotes a subgradient of $u \mapsto |u|$.
When $\lambda$ is larger than $\lambda_{\max} = \max_{j \in [p]} \sum_{i=1}^n |x_{ij}|$,  an optimal solution to 
Problem~\eqref{l1-SVM} is zero: $\B{\beta}^*(\lambda) = \mathbf{0}.$

Let $\mathcal{I}_+, \ \mathcal{I}_-$ denote the sample indices corresponding to the classes with labels $+1$ and $-1$ (respectively); and let $N_+, \ N_-$ denote their respective sizes. 
If  $N_+ \ge N_-$, then for $\lambda\ge \lambda_{\max}$  a solution to Problem \eqref{l1-SVM-dual} is $\pi_i(\lambda) = N_-/N_+, \forall i \in  \mathcal{I}_+$ and $\pi_i(\lambda) = 1, \forall i \in  \mathcal{I}_-$. For 
$\lambda = \lambda_{\max}$ using~\eqref{reduced-cost}, the minimum of the reduced costs of the variables $\beta^+_j$ and $\beta^-_j$  is
\begin{equation} \label{reduced-cost-primal}
\min \left\{ \bar{\beta}^{+}_j(\lambda_{\max}),  \bar{\beta}^{-}_j(\lambda_{\max}) \right\} = \lambda_{\max} - \left|  \frac{N_-}{N_+} \sum \limits_{i \in  \mathcal{I}_+} y_i x_{ij} + \sum \limits_{i \in  \mathcal{I}_-} y_i x_{ij} \right|.
\end{equation}
When $\lambda=\lambda_1$ is slightly smaller than $\lambda_{\max}$, we initialize column generation by selecting $\mathcal{J}$ as a small subset of variables that minimize the right-hand side of~\eqref{reduced-cost-primal}. 
Once we obtain a solution to Problem~\eqref{l1-SVM} at $\lambda_1$, we 
compute a solution for a smaller value of $\lambda$ by using LP warm-start along with column generation.
Consequently, this leads to solutions for a grid of $\lambda$ values, as
summarized below.

\medskip 
\medskip

\noindent \textsc{Algorithm~2:} {\bf{Regularization path algorithm for L1-SVM}}

\noindent \textbf{Input:} $\mathbf{X}$, $\mathbf{y}$, convergence tolerance $\epsilon$ (for column generation), a grid of decreasing $\lambda$ values: $\{\lambda_{0} = \lambda_{\max},\ldots,\lambda_{M}=\lambda \}$, a small integer $j_0$. \newline
\textbf{Output: } A sequence of solutions  $\left\{  \B{\beta}^*(\lambda_0),\ldots, \B{\beta}^*(\lambda_M)  \right\}$
on a grid of $\lambda$-values for the L1-SVM Problem \eqref{l1-SVM}.
\begin{enumerate} 
	\item Let $\B{\beta}^*(\lambda_0) = \M{0}$ and assign $\mathcal{J}(\lambda_0)$ to the $j_0$ variables minimizing the rhs of~\eqref{reduced-cost-primal}.
	
	\item For $\ell \in\left\{1,\ldots,M \right\}$ initialize $\mathcal{J}(\lambda_{\ell}) \leftarrow \mathcal{J}(\lambda_{\ell-1}), \ \B{\beta}^*(\lambda_{\ell}) \leftarrow \B{\beta}^*(\lambda_{\ell-1})$. Run column generation to obtain the new estimate $\B{\beta}^*(\lambda_{\ell})$ with $\mathcal{J}(\lambda_{\ell})$ denoting the corresponding columns.
\end{enumerate}

\subsection{Column and constraint generation for L1-SVM}\label{sec: CCG}

Moving beyond the large $p$ and small $n$ setting, we first consider the case where $n$ is large but $p$ is small (Section~\ref{sec: problemConG}).
{For the L1-SVM Problem~\eqref{l1-SVM}, if the linear classifier 
is able to correctly classify most of the samples, the corresponding terms in the hinge loss i.e., $\xi_{i}$'s will be zero at an optimal solution, and a constraint generation approach will be useful\footnote{If the underlying dataset cannot be well classified via a linear SVM at the corresponding sparsity level of $\B\beta$, then the constraint generation procedure may not be effective.}.} 
Section~\ref{sec:CCG-L1} discusses the case where both $n,p$ are large where we use both column and constraint generation.

\subsubsection{Constraint generation for large $n$ and small $p$}\label{sec: problemConG}
Given $\lambda>0$ and  $\mathcal{I} \subset \left\{1,\ldots,n \right\}$,
we define the \emph{restricted constraints} version of the L1-SVM problem, by using a subset (indexed by $\mathcal{I}$) of the constraints\footnote{Note that in the objective in~\eqref{l1-SVM-restricted-dual}, the first term involves a summation over $\sum_{i \in \mathcal I} \xi_{i}$---this is because, if  a constraint is not present, the corresponding value of $\xi_{i}$ will be zero.} in Problem \eqref{l1-SVM-primal-main}

	\begin{equation}\label{l1-SVM-restricted-dual}
	\begin{myarray}[1.1]{c c c c r}
	\boxed{\mathcal{P}_{\lam}(\mathcal{I}, [p] )}~~~~~~~~~~~ & \min \limits_{ \substack{\B{\xi} \in \mathbb{R}^{| \mathcal{I}  | },  \beta_0  \in \mathbb{R} \\
			\  \B{\beta}^+, \ \B{\beta}^- \in \mathbb{R}^{p } } }  &  \sum \limits_{i \in \mathcal{I} } \xi_i  + \lambda \sum \limits_{j =1}^p \beta^+_j + \lambda \sum \limits_{j=1}^p \beta^-_j&   \\
	~~~~~~~&\sbt & \;\; \ \xi_i +  \sum \limits_{j=1}^p  y_i x_{ij} \beta_j^+ - \sum \limits_{j=1}^p  y_i x_{ij} \beta_j^-  + y_i \beta_0 \ge 1 & \;\;\; i \in  \mathcal{I}  \\
	~~~~~~~&& \;\;\;\; \ \B{\xi} \ge 0, \ \B{\beta}^+ \ge 0, \ \B{\beta}^- \ge 0.& 
	\end{myarray}
	\end{equation}
	A dual of Problem \eqref{l1-SVM-restricted-dual} is given by:
	\begin{equation}\label{l1-SVM-restricted-dual1}
	\begin{myarray}[1.1]{c c c c r}
	\boxed{\mathcal{D}_{\lam}(\mathcal{I}, [p] )}~~~~~~~~~~~ 
	& \max \limits_{\B{\pi} \in \mathbb{R}^{| \mathcal{I}  | } } & \sum \limits_{i\in\mathcal{I} } \pi_i  & \\
	&\sbt& \;\;\;\; \ -\lambda \le \sum \limits_{i \in \mathcal{I}} y_i x_{ij} \pi_i \le \lambda   & \;\;\; j \in[p] \\
	&& \;\;\;\; \  \sum \limits_{i \in \mathcal{I}} y_i \pi_i = 0 &\\
	&& \;\; 0\le \pi_i \le 1 & \;\;\; i \in \mathcal{I}.
	\end{myarray}
	\end{equation}

Let $(\B{\beta}^{\dagger}, \beta_0^{\dagger}) \in \mathbb{R}^{p+1}$ and  $\{{\pi}^{\dagger}_i\}_{i\in\cI} $ denote optimal solutions of Problem~\eqref{l1-SVM-restricted-dual} and~\eqref{l1-SVM-restricted-dual1} (respectively). 
Note that a constraint in~\eqref{line-1-l1svm-const} corresponding to $i \notin \mathcal{I}$ is violated if 
\begin{eqnarray}\label{pi-reduced cost}
\bar{\pi}_{i} = 1 -  y_i ( \mathbf{x}_i^T  \B{\beta}^{\dagger} + \beta_0^{\dagger} )
\end{eqnarray}
is (strictly) positive. 
We add to $\mathcal{I}$  all (or a subset of) indices corresponding to the dual variables having reduced cost higher than a threshold $\epsilon \geq 0$ (specified a-priori).
We solve the new LP formed with the expanded set $\mathcal I$. The constraint generation algorithm for L1-SVM is summarized below:

\clearpage 

\noindent \textsc{Algorithm~3}: \textbf{Constraint generation for L1-SVM}

\noindent \textbf{Input:} $\mathbf{X}$, $\mathbf{y}$, regularization parameter  $\lambda$, a tolerance threshold $\epsilon \geq 0$, an initial set of constraints indexed by $\mathcal{I}$.

\noindent \textbf{Output:} A near-optimal solution $\B{\beta}^{\dagger}$ for the L1-SVM Problem \eqref{l1-SVM}.
\begin{enumerate}
	\item Repeat Steps 2 to 3 until $\mathcal{I}$ stabilizes.
	
	\item Solve Problem {\color{black} $\mathcal{P}_{\lam}(\mathcal{I}, [p])$} (i.e., Problem \eqref{l1-SVM-restricted-dual}).
	
	\item Let $\mathcal{I}^{\epsilon} \subset \left\{1,\ldots,n \right\} \backslash \mathcal{I}$ denote constraints with violations larger than $\epsilon$. 
	Update  $\mathcal{I} \leftarrow \mathcal{I} \cup \mathcal{I}^{\epsilon}$, and go to Step~2 (with LP warm-starting enabled).
\end{enumerate}

The following theorem gives an upper bound on the resulting objective value upon termination of Algorithm~3. 
	\begin{thm}\label{thm-2-statement}
		Let $z^*$ denote the optimal objective value of $\cP_{\lam}([n], [p])$. Let $(\B{\beta}^{\dagger}, \beta_0^{\dagger}) \in \mathbb{R}^{p+1}$ denote an optimal solution of Problem~\eqref{l1-SVM-restricted-dual}; and 
		\begin{eqnarray}\label{z-dagger}
		z^{\dagger}  := \sum \limits_{i=1}^n \left( 1 - y_i( \mathbf{x}_i^T \B{\beta}^{\dagger} + \beta_0^{\dagger}   ) \right)_+  + \lambda \| \B{\beta}^{\dagger}  \|_1. 
		\end{eqnarray}
		At an iteration of Algorithm~3, let $\td{\ep} := \max_{i\in [n]\backslash \cI}\{ \max\{\bar \pi_i,0\} \}$ with $\bar \pi_i$, $i\in [n]\backslash \cI$, defined in \eqref{pi-reduced cost}. Then it holds
		\begin{eqnarray}
		 z^* \le z^\dagger \le z^* + \td{\ep} (n-|\cI|).
		\end{eqnarray}
	\end{thm}
\begin{proof}
	The first inequality is trivial. For the second inequality, note that from \eqref{z-dagger} we have
	\begin{eqnarray}\label{ineq1}
	z^{\dagger} 
	=
	\sum_{i\in[n] \backslash \cI} \max(\bar \pi_i , 0) + \sum_{i\in  \cI} \left( 1 - y_i( \mathbf{x}_i^T \B{\beta}^{\dagger} + \beta_0^{\dagger}   ) \right)_+   + \lam \| \B{\beta}^{\dagger} \|_1 . 
	\end{eqnarray}
	By the definition of $\td\ep$ we know that 
	\begin{eqnarray}\label{ineq2}
	\sum_{i\in[n] \backslash \cI} \max(\bar \pi_i , 0) \le (n-|\cI|)\td\ep.
	\end{eqnarray}
Since $  \sum_{i\in  \cI} ( 1 - y_i( \mathbf{x}_i^T \B{\beta}^{\dagger} + \beta_0^{\dagger}   ) )_+  + \lam \| \B{\beta}^{\dagger} \|_1$ is the optimal value of $\cP_{\lam} (\cI, [p])$, and because $\cP_{\lam} (\cI, [p])$ is a relaxation of $\cP_{\lam} ([n], [p])$, we have 
\begin{eqnarray}\label{ineq3}
\sum_{i\in  \cI} \left( 1 - y_i( \mathbf{x}_i^T \B{\beta}^{\dagger} + \beta_0^{\dagger}   ) \right)_+   + \lam \| \B{\beta}^{\dagger} \|_1 \le z^*.
\end{eqnarray}
Combining \eqref{ineq1}, \eqref{ineq2} and \eqref{ineq3} we obtain the upper bound on $z^\dagger$.
\end{proof}
Note that the value $\td \ep \geq 0$ satisfies $\td \ep \le \ep$ where $\ep$ is the value of tolerance we set in Algorithm 3.
When $\td\ep=0$, we are at an optimal solution.

\paragraph{Initialization:} Similar to the case in column generation, the constraint generation procedure benefits from a
 good initialization scheme. To this end, we use first order methods as described in Section \ref{sec: FOM}---specifically, the method we use for large $n$ (and small $p$), is discussed in Section~\ref{sec: heuristic-aggregation}.

\subsubsection{Column and constraint generation when both $n$ and $p$ are large}\label{sec:CCG-L1}

When both $n$ and $p$ are large, we will use a combination of column and constraint generation to solve the L1-SVM problem. For a given $\lambda$, let 
$\mathcal{I}$ and $\mathcal{J}$ denote subsets of columns and constraints (respectively).
This leads to the following restricted version of the L1-SVM problem:

	\begin{equation}\label{l1-SVM-doubly-restricted-primal}
	\begin{myarray}[1.1]{l c c r}
	\boxed{\cP_{\lam}(\mathcal{I}, \mathcal{J})}~~ & \min \limits_{ \substack{\B{\xi} \in \mathbb{R}^{| \mathcal{I}  | },\ \beta_0  \in \mathbb{R} \\ \  \B{\beta}^+, \ \B{\beta}^- \in \mathbb{R}^{| \mathcal{J}| }  } }  &  \sum \limits_{i \in \mathcal{I} } \xi_i  + \lambda \sum \limits_{j \in \mathcal{J} } \beta^+_j + \lambda \sum \limits_{j \in \mathcal{J}} \beta^-_j& \\
	&\sbt & \;\; \ \xi_i +  \sum \limits_{j \in \mathcal{J} }  y_i x_{ij} \beta_j^+ - \sum \limits_{j \in \mathcal{J} }  y_i x_{ij} \beta_j^-  + y_i \beta_0 \ge 1 & \;\;\;  i \in  \mathcal{I}  \\
	&& \ \B{\xi} \ge 0, \ \B{\beta}^+ \ge 0, \ \B{\beta}^- \ge 0.  &
	\end{myarray}
	\end{equation}
	
	The dual problem of $\cP_\lam(\cI,\cJ)$ is 
	\begin{equation}\label{l1-SVM-doubly-restricted-dual}
	\begin{myarray}[1.1]{c c c c r}
	\boxed{\cD_{\lam}(\cI, \cJ)}:\;\;\;\;& \max \limits_{\B{\pi} \in \mathbb{R}^{|\cI| } } & \centermathcell{ \sum \limits_{i\in \cI} \pi_i }   \\
	&\;\;\; \sbt  & \centermathcell{ \;\;\; -\lambda \le  \sum \limits_{i\in \cI} y_i x_{ij} \pi_i \le \lambda \;\;\; }   & \;\;\;\;\; j \in \cJ \\ 
	& & \centermathcell{  \sum_{i\in \cI} y_i \pi_i = 0 }  \\
	& &  \centermathcell{ 0\le \pi_i \le 1 } & \;\;\;\;\; i \in \cI. 
	\end{myarray}
	\end{equation}

Let $(\{\hat{{\beta}}^{\dagger}_j\}_{j\in \cJ}, \hat\beta_0^{\dagger}),~ \{\hat{\pi}^{\dagger}_i\}_{i\in \cI}$ be a pair of optimal primal and dual solutions for the above problem. 
Let $\bar{\beta}^+_j, \bar{\beta}^-_j$ denote the reduced costs for primal variables $\beta^+_j, \beta^-_j$; and $\bar{\pi}_{i}$ denote the reduced cost for dual variable $\pi_{i}$. 
The reduced costs are given by:
\begin{equation} \label{reduced-costs}
\min \left\{ \bar{\beta}^{+}_j,  \bar{\beta}^{-}_j \right\} = \lambda - \left|   \sum \limits_{i\in\mathcal{I}}  y_i x_{ij} \hat\pi_i^{\dagger}  \right| ~\text{for}~j\notin \cJ; \ \  \ \ \bar{\pi}_i = 1 - y_i \left( \sum \limits_{j \in \mathcal{J}} x_{ij} \hat\beta_j^{\dagger}  + \hat\beta_0^{\dagger}   \right)~ \text{for}~i\notin \cI.
\end{equation}
We expand the sets ${\mathcal I}$ and $\mathcal J$ by using Steps 3 and 4 of Algorithm~4 (below). 
We then solve Problem \eqref{l1-SVM-doubly-restricted-primal} (with warm-starting enabled) and continue till ${\mathcal I}$ and $\mathcal J$ stabilize.
Section \ref{sec: FOM-CCG} discusses the use of first order optimization methods to initialize $\mathcal{I}$ and $\mathcal{J}$. 
Our hybrid column and constraint generation approach to solve the L1-SVM Problem \eqref{l1-SVM} is summarized below.

\medskip
\medskip
\medskip

\noindent \textsc{Algorithm~4}: \textbf{Combined column and constraint generation for L1-SVM} \newline
\textbf{Input:} $\mathbf{X}$, $\mathbf{y}$, a regularization coefficient  $\lambda$, tolerance thresholds $\epsilon_1, \epsilon_2 \geq 0$, initial subsets $\mathcal I$ and $\mathcal J$. \newline
\textbf{Output: }A near-optimal solution $\hat{\B{\beta}}^{\dagger}$ for the L1-SVM Problem \eqref{l1-SVM}.
\begin{enumerate}
	\item 
	Repeat Steps 2 to 4 until $\mathcal{I}$  and $\mathcal{J}$ stabilize.
	
	\item Solve $\mathcal{P}_{\lam}(\mathcal{I}, \mathcal{J})$ that is,  
	Problem~\eqref{l1-SVM-doubly-restricted-primal}.
	
	\item Let $\mathcal{I}^{\epsilon_1} \subset \left\{1,\ldots,n \right\} \backslash \mathcal{I}$ denote constraints with reduced cost higher than $\epsilon_1$.
	Update  $\mathcal{I} \leftarrow \mathcal{I} \cup \mathcal{I}^{\epsilon_1}$.
	
	\item Let $\mathcal{J}^{\epsilon_2} \subset \left\{1,\ldots,p \right\} \backslash \mathcal{J}$ denote columns with reduced cost lower than $-\epsilon_2$.  
	Update  $\mathcal{J} \leftarrow \mathcal{J} \cup \mathcal{J}^{\epsilon_2}$; and go to Step 2.
\end{enumerate}

The following theorem gives an optimality certificate of a solution obtained from Algorithm~4. 
	\begin{thm}
		Let $z^*$ denote the optimal objective value of $\cP_{\lam}([n], [p])$. 
		Recall that $(\{\hat{{\beta}}^{\dagger}_j\}_{j\in \cJ}, \hat\beta_0^{\dagger})$ is an optimal solution of Problem $\cP_{\lam} (\cI, \cJ)$. 
		Extend $\{\hat{{\beta}}^{\dagger}_j\}_{j\in \cJ}$ into a vector $\hat{\B{\beta}}^{\dagger} \in \R^p$ with $\hat \beta_j^{\dagger} = 0$ for all $j\in [p]\backslash\cJ$. 
		Denote
		\begin{eqnarray}
		\hat z^{\dagger}  := \sum \limits_{i=1}^n \left( 1 - y_i( \mathbf{x}_i^T \hat{\B{\beta}}^{\dagger} + \hat\beta_0^{\dagger}   ) \right)_+  + \lambda \| {\hat{\B{\beta}}}^{\dagger}  \|_1. 
		\end{eqnarray}
		At an iteration of Algorithm~4, let 
		\begin{eqnarray}
		\td\ep_1 := \max_{i\in [n]\backslash \cI}\{\max(\bar \pi_i , 0)\}, ~~~ \td\ep_2 := - \min_{j\in [p] \setminus \cJ} \lt\{ \min \{\bar \beta_j^+ ,  \bar \beta_j^- ,0  \}  \rt\}
		\end{eqnarray}
		where $ \min \{ \bar{\beta}^{+}_j,  \bar{\beta}^{-}_j \}$ and $\bar \pi_i$ are defined in \eqref{reduced-costs}.
If $\B\beta^*$ is an optimal solution to $\cP_{\lam}([n], [p])$, it holds
		\begin{equation*}
		z^* \le \hat z^{\dagger} \le z^* + \td\ep_1 (n-|\cI|) + \td\ep_2 \| \B{\beta}^*\|_1.
				\end{equation*}
	\end{thm}
\begin{proof}
	The first inequality is trivial. Below we prove the second inequality. For any $i\in [n] \backslash \cI$ let $ \hat \pi_i^\dagger =0$. Then from the definition of $\td\ep_2$ we have 
	\begin{eqnarray}
	\lam - \left|   \sum \limits_{i=1}^n  y_i x_{ij} \hat\pi_i^{\dagger}  \right| =
	\lam - \left|   \sum \limits_{i\in\mathcal{I}}  y_i x_{ij} \hat\pi_i^{\dagger}  \right|  \ge -\td\ep_2
	 ~~~~ \forall j\in [p] \setminus \cJ. 
	\end{eqnarray}
	On the other hand, since 
	\begin{eqnarray}
	\lam - \left|   \sum \limits_{i=1}^n  y_i x_{ij} \hat\pi_i^{\dagger}  \right| =
	\lam - \left|   \sum \limits_{i\in\mathcal{I}}  y_i x_{ij} \hat\pi_i^{\dagger}  \right|  \ge 0~~~~ \forall j\in \cJ,
	\end{eqnarray}
	we have 
	\begin{eqnarray}
	-\lam - \td\ep_2 \le  \sum \limits_{i=1}^n  y_i x_{ij} \hat\pi_i^{\dagger} \le \lam + \td\ep_2  ~~~~ \forall j\in [p]. 
	\end{eqnarray}
	In addition, we have $ \sum_{i=1}^n y_i \hat\pi_i^\dagger =  \sum_{i\in \cI} y_i \hat\pi_i^\dagger =0$ and $0\le  \hat\pi_i^\dagger  \le 1$ for all $i\in [n]$. So $\{\hat \pi_i^\dagger\}_{i=1}^n$ is a solution of $\cD_{\lam+\td\ep_2} ([n], [p])$. Let $ (\B{\xi}^*, \B{\beta}^{*+}, \B{\beta}^{*-}, \beta^*_0) $ be an optimal solution of $\cP_{\lam} ([n], [p])$, then it is also a feasible solution of $\cP_{\lam+\td\ep_2} ([n], [p])$. Hence by weak duality we have 
	\begin{eqnarray}\label{ineq-1}
	\sum_{i=1}^n \hat\pi_i^\dagger \le \sum_{i=1}^n \xi_i^* +(\lam+\td\ep_2 ) \sum_{j=1}^p (\beta_j^{*+} + \beta_j^{*-}) = z^* + \td\ep_2 \| \B{\beta}^* \|_1. 
	\end{eqnarray}
		On the other hand, from the definition of $\hat z^\dagger$ we have 
	\begin{eqnarray}\label{ineq-2}
	\hat z^\dagger &=& \sum \limits_{i\in \cI } \left( 1 - y_i( \mathbf{x}_i^T \hat{\B{\beta}}^{\dagger} + \hat \beta_0^{\dagger}   ) \right)_+  + \sum_{i\in [n]\backslash \cI }  \max(\bar \pi_i , 0) + \lambda \|  \hat{\B{\beta}}^{\dagger}  \|_1 \nonumber\\
	&\le&
	 \sum \limits_{i\in \cI } \left( 1 - y_i( \mathbf{x}_i^T \hat{\B{\beta}}^{\dagger} + \hat \beta_0^{\dagger}   ) \right)_+  + \td\ep_1 (n - |\cI|)  +  \lambda \|  \hat{\B{\beta}}^{\dagger}  \|_1 \nonumber\\
	&=&
		 \sum \limits_{i\in \cI } \left( 1 - y_i( \mathbf{x}_i^T \hat{\B{\beta}}^{\dagger} + \hat \beta_0^{\dagger}   ) \right)_+  + \td\ep_1 (n - |\cI|) +\lam 
		\sum_{j\in \cJ} |\hat \beta_j^\dagger|
	\end{eqnarray}
	where in the inequality above, we use the definition of $\tilde{\epsilon}_{2}$.\\
	Since $\sum \limits_{i\in \cI } \left( 1 - y_i( \mathbf{x}_i^T \hat{\B{\beta}}^{\dagger} + \hat \beta_0^{\dagger}   ) \right)_+  + \lam 
	\sum_{j\in \cJ} |\hat \beta_j^\dagger|$ is the optimal objective value of $ \cP_\lam (\cI, \cJ) $, it holds (by strong duality)
	\begin{eqnarray}\label{ineq-3}
	\sum \limits_{i\in \cI } \left( 1 - y_i( \mathbf{x}_i^T \hat{\B{\beta}}^{\dagger} + \hat \beta_0^{\dagger}   ) \right)_+  + \lam 
	\sum_{j\in \cJ} |\hat \beta_j^\dagger| ~=~
	\sum_{i\in \cI} \hat\pi_i^\dagger 
	= 	\sum_{i=1}^n \hat\pi_i^\dagger .
	\end{eqnarray}
	Combining \eqref{ineq-1}, \eqref{ineq-2} and \eqref{ineq-3} we complete the proof. 
\end{proof}
Note that the nonnegative values $\td \ep_1$ and $\td \ep_2$ satisfy $\td \ep_1  \le \ep_1$ and $\td \ep_2  \le \ep_2$  where $\ep_1$ and $\ep_2$ are tolerances we set in Algorithm 4.
When $\td\ep_1 = \td\ep_2=0$, we are at an optimal solution.

\subsection{Application to the Group-SVM problem}\label{sec: group-CG}

We now discuss application of column/constraint generation to the Group-SVM Problem~\eqref{group-SVM-intro}.
We first introduce some notation that we will use---we let $\mathcal{I}_g\subset[p]$ denote the indices corresponding to group $g$ for $g \in[G]$.

\smallskip 

\noindent {\bf Column generation:} Below we present an LP formulation for Problem \eqref{group-SVM-intro}. We introduce the variables $\B{v}=(v_{g})_{g \in [G]}$ such that $v_g$ refers to the $L_{\infty}$-norm of the coefficients $\B\beta_{g}$:
\begin{equation} \label{group-SVM-primal}
\begin{myarray}[1.1]{l c c l}
\text{(Group-SVM)}& \min \limits_{ \substack{\B{\xi} \in \mathbb{R}^n,  \beta_0 \in \mathbb{R},\\  \B{\beta}^+,\  \B{\beta}^- \in \mathbb{R}^p, \B{v} \in \mathbb{R}^G }   } &  \sum \limits_{i=1}^n \xi_i  + \lambda \sum \limits_{g=1}^G v_g &\\
&\sbt & \;\; \ \xi_i + y_i \mathbf{x}_i^T \B{\beta}^+ - y_i \mathbf{x}_i^T \B{\beta}^-  + y_i \beta_0 \ge 1 & \;\;\; i \in [n] \\
&& \;\; \ v_g - \beta_j^+ - \beta_j^- \ge 0& j \in \mathcal{I}_g, \ g \in [G]   \\
&& \;\;\;\; \ \B{\xi} \ge 0, \ \B{\beta}^+ \ge 0, \ \B{\beta}^- \ge 0, \  \B{v}  \ge 0. &\\
\end{myarray}
\end{equation}
A dual of Problem \eqref{group-SVM-primal} is given by:
\begin{equation}\label{group-SVM-dual}
\begin{myarray}[1.1]{l c c c}
\text{(Dual-Group-SVM)}~~& \max \limits_{\B{\pi} \in \mathbb{R}^{n } } & \sum \limits_{i=1}^n \pi_i  & \\
& \sbt  & \;\;\;\; \ \sum \limits_{ j \in \mathcal{I}_g} \left| \sum \limits_{i=1}^n  y_i x_{ij} \pi_i \right| \le \lambda   & \;\;\; g \in [G] \\
& & \  \mathbf{y}^T \B{\pi} = 0 &\\
& &  0\le \pi_i \le 1 & i \in [n].
\end{myarray}
\end{equation}
Following the description in Section~\ref{sec: CG-LP}, we apply column generation on the groups.
Here, the reduced cost  of group $g$ is given as:
\begin{equation} \label{reduced-costs-group}
\bar{{\beta}_g} = \lambda - \sum_{j \in \mathcal{I}_g } \left|   \sum \limits_{i=1}^n  y_i x_{ij} \pi_i  \right|
\end{equation}
and we include into the model (a subset of) groups $g$ for which 
$\bar{{\beta}_g} <0$ (for a small negative tolerance).

\paragraph{Computing a regularization path: } 
The regularization path algorithm presented in Section~\ref{sec: correlation}
can be adapted to the Group-SVM problem.  First, note that:
\begin{equation}\label{lam-max-defn-grp}
\B{\beta}^*(\lambda) = \mathbf{0}, \;\;\;\; \forall \lambda \ge \lambda_{\max} = \max_{g \in [G]} \sum_{j \in \mathcal{I}_g} \sum_{i=1}^n |x_{ij}|.
\end{equation}
For $\lambda = \lambda_{\max}$, the reduced cost of variables corresponding to group $g$ is given by the ``group'' analogue of~\eqref{reduced-cost-primal}:
\begin{equation}\label{reduced-costs-groups}
\bar{{\beta}_g}  = \lambda_{\max} - \sum_{j \in \mathcal{I}_g} \left|  \frac{N_-}{N_+} \sum \limits_{i \in  \mathcal{I}_+} y_i x_{ij} + \sum \limits_{i \in  \mathcal{I}_-} y_i x_{ij} \right|.
\end{equation}
As in Section~\ref{sec: correlation}, we can obtain a small set of groups maximizing the rhs of~\eqref{reduced-costs-groups}. We use these groups to initialize 
the LP solver to solve Problem~\eqref{group-SVM-primal} for the next small value of $\lambda$, using column generation---this results in computational savings when the number of active groups is small compared to $G$.
We repeat this process for smaller values of $\lambda$ using warm-start continuation.

\smallskip 

\noindent {\bf Constraint generation and column generation: } When $n$ is large (but the number of groups is small) constraint generation can be used for 
the Group-SVM problem in a manner similar to that used for the L1-SVM problem. Similarly, column and constraint generation can be applied together to obtain computational savings when both $n$ and the number of groups are large. 

\smallskip 

\noindent {\bf Initialization:} As discussed for the L1-SVM problem, we can use first order methods to obtain a low-accuracy solution for the Group-SVM problem---these are discussed in Section~\ref{sec: FOM}. 
This can be used to initialize the set of nonzero groups (for column generation), relevant constraints (for constraint generation), and both groups and constraints (for the combined column-and-constraint generation procedure).

\section{Column and constraint generation for Slope-SVM}\label{sec: group-L1-SVM}
Here we discuss the Slope-SVM estimator i.e., Problem~\eqref{Slope-intro}.
For a $p$-dimensional regularization parameter $\B\lambda$ with coordinates sorted as: $\lambda_1 \ge \ldots \ge \lambda_p \ge0$, we let
\begin{equation} \label{slope-norm}
\| \B{\beta} \|_S := \sum_{j=1}^p \lambda_j | \beta_{(j)} |
\end{equation}
denote the Sorted L1-norm or the \emph{Slope-norm}---we borrow this term inspired by the ``Slope estimator''~\cite{slope-introduction} and acknowledge a slight abuse in terminology. Note that for convenience, we drop the dependence on $\B\lambda$ in the notation $\| \cdot \|_{S}$.

We note that the epigraph of $\| \B{\beta} \|_S$ i.e., $\{(\B\beta, \eta)~ |~ \| \B\beta \|_{S} \leq \eta\}$ admits an LP formulation using $O(p^2)$ many variables and $O(p^2)$ many constraints (cf Section~\ref{epi-slope-1})---given the problem-sizes we seek to address, we do not 
pursue this route. Rather, we make use of the observation that the epigraph 
can be expressed with exponentially many linear inequalities (Section~\ref{con-gen-slope}) when using 
$O(p)$-many variables. 
The large  number of constraints make column/constraint generation methodology for the Slope penalty significantly different than the L1-SVM example. However, as we will discuss, our procedure does \emph{not} require us to explicitly enumerate the exponentially many inequalities. 
To our knowledge, the procedure we present here for the Slope-SVM problem is novel. Section~\ref{con-gen-slope} discusses a constraint generation method that greatly reduces the number of constraints needed
to model the epigraph. Section~\ref{sec:slope-columns}  discusses the use of column generation to 
exploit 
sparsity in $\B\beta$ when $p$ is large. Finally, 
Section \ref{sec:slope-CCG} combines these two features to address the Slope SVM problem.
We note that even for large $p$ and small $n$ settings, both column and constraint generation methods are needed for the Slope penalty, making it different 
from the L1-penalty, where column generation (alone) suffices. 
In what follows, we concentrate on the case where $n$ is small but $p$ is large --- if $n$ is also large, a further layer of constraint generation might be needed to efficiently handle sparsity arising from the hinge-loss.

\subsection{Constraint generation for Slope-SVM}\label{con-gen-slope}

\noindent {\bf Reformulation of Slope-SVM:}
Note that Problem \eqref{Slope-intro} can be expressed as: 
\begin{subequations}\label{slope-primal}
	\begin{alignat}{6}
	\boxed{\mathcal{M}_{S}(\mathcal{C}, [p])} ~~~~~~~ & \min \limits_{ \substack{ \B{\xi} \in \mathbb{R}^n,  \ \beta_0,  \eta \in \mathbb{R}, \\  \  \B{\beta}^+,\  \B{\beta}^- \in \mathbb{R}^p }  } &  \centermathcell{\sum \limits_{i=1}^n \xi_i  + \eta} & \nonumber \\
	& \sbt\; &  \centermathcell{  \xi_i + y_i \mathbf{x}_i^T \B{\beta}^+ - y_i \mathbf{x}_i^T \B{\beta}^-  + y_i \beta_0 \ge 1 }& \;\;\;\; i \in [n]  \label{slope-primal-eq0} \\
	&& \centermathcell{ (\B{\beta}^+,\  \B{\beta}^- , \ \eta) \in \mathcal{C} } & \label{slope-primal-eq} \\
	&& \centermathcell{  \B{\xi} \ge 0, \ \B{\beta}^+ \ge 0, \ \B{\beta}^- \ge 0} & \nonumber
	\end{alignat}
\end{subequations}
where, $\B{\beta} = \B{\beta}^+ - \B{\beta}^-$; and  $\B{\beta}^+, \ \B{\beta}^- \in \R^p_+$ denote the positive and negative parts of $\B\beta$ (respectively). 
In line~\eqref{slope-primal-eq} we express the Slope penalty in the epigraph form with $\mathcal{C}$ defined as:
$$ \mathcal{C} := \left\{  (\B{\beta}^+,\  \B{\beta}^- , \ \eta) \ \biggr\rvert \ \eta \ge \sum \limits_{j=1}^p \lambda_j \beta^+_{(j)} + \sum \limits_{j=1}^p \lambda_j \beta^-_{(j)}, ~~ \B{\beta}^+,\  \B{\beta}^-\in\mathbb{R}^p_+  \right\}$$
where, we use the notation $\beta^+_{(1)} + \beta^-_{(1)} \ge \ldots \ge \beta^+_{(p)} + \beta^-_{(p)}$ and remind ourselves that $|\beta_{i}| = \beta_{i}^+ + \beta_{i}^{-}$ for all $i$.
Below we show that~\eqref{slope-primal-eq} can be expressed via linear inequalities involving $(\B\beta^+,\B\beta^{-})$.

We first introduce some notation. Let $\mathcal{S}_p$ denote the set of all permutations of $\left\{ 1, \ldots, p \right\}$, with $|\mathcal{S}_p| = p!$. For a permutation $\phi \in {\mathcal S}_{p}$, we let 
$(\phi(1), \ldots, \phi(p))$ denote the corresponding rearrangement of $(1, \ldots, p)$. Using this notation, 
the Slope norm can be expressed as:
\begin{equation} \label{slope-reformulation}
\| \B{\beta} \|_S = \sum_{j=1}^p \lambda_j | \beta_{(j)} | = \max \limits_{\phi \in \mathcal{S}_p } \sum_{j=1}^p  \lambda_j   | \beta_{\phi(j)} | = \max \limits_{\psi \in \mathcal{S}_p } \sum_{j=1}^p  \lambda_{\psi(j)}  | \beta_{j} |.
\end{equation}
As a consequence, we have the following lemma:
\begin{lemma}\label{slope-lemma}
	The Slope norm $\| \B\beta \|_S$ admits the following representation
		$$\| \B{\beta} \|_S =   \max \limits_{  \B{w} \in \mathcal{W}^{[p]} } \B{w}^T  (\B{\beta}^+ + \B{\beta}^-) = \max_{\B{w} \in \mathcal{W}_0^{[p]}  }  \B{w}^T  (\B{\beta}^+ + \B{\beta}^-)$$ 
	where, $\mathcal{W}_0^{[p]} :=  \Conv \left( \mathcal{W}^{[p]}   \right)$ is the convex hull of $\mathcal{W}^{[p]}$, where
	\begin{equation}\begin{aligned}
	\label{eq-wp}
 \mathcal{W}^{[p]} := \left\{\B{w} \in \mathbb{R}^p  ~ \rvert \  \exists \psi \in \mathcal{S}_p ~~\text{s.t.}~~ \  w_j = \lambda_{\psi(j)}, j \in [p] \right\}.
	\end{aligned} \end{equation}
\end{lemma}
\begin{proof}
Note that a linear function maximized over a bounded polyhedron reaches its maximum at one of the extreme points of the polyhedron --- this leads to:
	\begin{equation}\label{proof-slope-form-1}
	\max \limits_{\B{w} \in \mathcal{W}_0^{[p]}  }  \B{w}^T  (\B{\beta}^+ + \B{\beta}^-)
	= \max \limits_{  \B{w} \in \mathcal{W}^{[p]} } \B{w}^T  (\B{\beta}^+ + \B{\beta}^-).
\end{equation}
	 Using the definition of $\mathcal{W}^{[p]}$, we get that the rhs of~\eqref{proof-slope-form-1} is 
	 $\max \nolimits_{\psi \in \mathcal{S}_p } \sum \nolimits_{j=1}^p  \lambda_{\psi(j)}  | \beta_{j} |$ which is in fact the Slope norm $\| \B{\beta} \|_S.$
\end{proof}

The following remark provides a description of $\mathcal{W}^{[p]}$ for special choices of $\B\lambda$.
\begin{remark} 
 \textbf{(a)} If all the coefficients are equal i.e., $\lambda_1=\ldots=\lambda_p$ and 
$\|\B\beta\|_{S} = \lambda \| \B\beta \|_{1}$, then $\mathcal{W}^{[p]}$ is a singleton. 
\textbf{(b)} 
If all the coefficients are distinct i.e., $\lambda_1 > 
\ldots > \lambda_p$, then each permutation $\psi \in \mathcal{S}_p$ is associated with a unique vector in $\mathcal{W}^{[p]}$ and $\mathcal{W}^{[p]}$ contains $p!$ elements.
\end{remark}
Using  Lemma \ref{slope-lemma}, we can derive an LP formulation of Problem \eqref{slope-primal} by modeling 
$\mathcal{C}$ in~\eqref{slope-primal-eq} as: 
\begin{equation}\label{C-definition}
\mathcal{C} = \left\{  \left(\B{\beta}^+,\  \B{\beta}^- , \ \eta \right) ~\bigg\rvert~ \B{\beta}^+, \B{\beta}^- \in \mathbb{R}^p,~~  \eta \ge \max \limits_{  \B{w} \in \mathcal{W}^{[p]} } \B{w}^T  (\B{\beta}^+ + \B{\beta}^-) \right\},
\end{equation}
where, $\mathcal{W}^{[p]}$ is defined in~\eqref{eq-wp}.
The resulting LP formulation~\eqref{slope-primal} has at most $n$ constraints from~\eqref{slope-primal-eq0} and at most $p!$ constraints associated with~\eqref{slope-primal-eq} (by virtue of~\eqref{C-definition}). 
We note that many constraints in~\eqref{C-definition} are redundant: for example, the maximum is attained corresponding to the inverse of permutation $\phi$ (denoted by $\phi^{-1}$), where $| \beta_{\phi(1)} | \ge \ldots \ge | \beta_{ \phi(p) } | $. This motivates the use of constraint generation techniques.

\smallskip

\noindent {\bf Constraint generation:} We proceed by replacing $\mathcal{W}^{[p]}$ with a smaller subset 
and solve the resulting LP. 
We subsequently refine this approximation 
 if~\eqref{slope-primal-eq} is violated.
Formally, let us consider a collection of vectors/cuts $\M{w}^{(1)}, \ldots, \M{w}^{(t)} \in \mathcal{W}^{[p]}$
leading to a superset $\mathcal{C}_t$ of $\mathcal{C}$:
\begin{equation}\label{slope-primal-restricted}
\mathcal{C} \subseteq \mathcal{C}_t := \left\{  (\B{\beta}^+,\  \B{\beta}^- , \ \eta) \ \biggr\rvert ~ \B{\beta}^+, \B{\beta}^- \in \mathbb{R}^p,~~  \eta \ge \sum \limits_{j=1}^p w^{(\ell)}_j \beta^+_{j} + \sum \limits_{j=1}^p w^{(\ell)}_j \beta^-_{j}, \ \forall \ell \le t \right\}. 
\end{equation} 
By replacing $\mathcal{C}$ in~\eqref{slope-primal-eq} by $\mathcal{C}_t$, we
get an LP denoted by $\mathcal{M}_{S}(\mathcal{C}_t, [p])$ which is a relaxation of $\mathcal{M}_{S}(\mathcal{C}, [p])$.
Let $\left(\B{\beta}^{*}, \eta^*\right)$ be a solution of $\mathcal{M}_{S}(\mathcal{C}_t, [p])$. If this is not an optimal solution (at a tolerance threshold $\epsilon \geq 0$),
we add a cut to $\mathcal{C}_t$ if 
$$\eta^* + \epsilon < \sum \limits_{j=1}^p \lambda_j | \beta^*_{(j)} |  = \| \B\beta^*\|_{S}.$$ 
To this end, consider a permutation $\psi_{t+1} \in \mathcal{S}_p$ such that  $|\beta^{*}_{\psi_{t+1}(1)}| \ge \ldots \ge |\beta^{*}_{\psi_{t+1}(p)}| $. If  $\psi_{t+1}^{-1}$ denotes the inverse of $\psi_{t+1}$, we obtain $\M{w}^{(t+1)} \in \mathcal{W}^{[p]}$ such that:
\begin{equation}\label{adding-cut-slope}
{w}^{(t+1)}_{ j } = \lambda_{\psi^{-1}_{t+1}(j)} \ \forall j \in [p]
\end{equation}
and solve the resulting LP. 
We continue adding cuts, till no further cuts need to be added --- this leads to a (near)-optimal solution to $\mathcal{M}_{S}(\mathcal{C}, [p])$.
We note that the first cut $\M{w}^{(1)}$ can be obtained by applying~\eqref{adding-cut-slope} on an estimator obtained from the first order optimization schemes (cf Section \ref{sec: FOM}). 
Our algorithm is summarized below for convenience.

\medskip
\medskip
\medskip

\noindent \textsc{Algorithm~5:}  \textbf{Constraint generation for Slope-SVM} \newline
\textbf{Input:} $\mathbf{X}$, $\mathbf{y}$, a vector of Slope coefficients $\left\{ \lambda_j \right\}_{j \in [p]}$, a tolerance threshold $\epsilon>0$, a cut $\M{w}^{(1)}\in \mathcal{W}^{[p]}$. 

\noindent \textbf{Output: }A near-optimal solution $\B{\beta}^*$ for the Slope-SVM Problem \eqref{Slope-intro}.
\begin{enumerate}
	\item 
	Repeat Steps 2 to 3 (for $t \geq 1$) till no further cuts need to be added.
	
	\item Solve the model $\mathcal{M}_{S}(\mathcal{C}_{t},[p])$ with $\mathcal{C}_{t}$ as in~\eqref{slope-primal-restricted}. Let $\left(\B{\beta}^{*}, \eta^* \right)$ be a solution.
	
	\item Let $\psi_{t+1} \in \mathcal{S}_p$ be such that $|\beta^{*}_{\psi_{t+1}(1)}| \ge \ldots \ge |\beta^{*}_{\psi_{t+1}(p)}|$. If condition $\eta^* + \epsilon \ge \sum_{j=1}^p \lambda_j | \beta^{*}_{\psi_{t+1}(j)}| $ is not satisfied, we 
	add a new cut $\M{w}^{(t+1)} \in \mathcal{W}^{[p]} $ as per~\eqref{adding-cut-slope}; update $\mathcal{C}_{t+1}$ and go to Step~2.
\end{enumerate}

\subsection{Dual formulation and column generation for Slope-SVM}\label{sec:slope-columns}

When the amount of regularization is high, the Slope penalty (with $\lambda_{i}>0$ for all $i$) may lead to many zeros in an optimal solution to  Problem~\eqref{Slope-intro} --- computational savings are possible if we can leverage this sparsity when 
$p$ is large. To this end, we use column generation along with the  constraint generation algorithm described in Section~\ref{con-gen-slope}. In particular, given a set of columns $\mathcal{J} = \left\{ \mathcal{J}(1), \ldots, \mathcal{J}(| \mathcal{J} |)  \right\} \subset [p]$, we consider a restricted version of Problem~\eqref{Slope-intro} with $\beta_{j} = 0, j \notin {\mathcal J}$: 
$$ \min \limits_{ \B{\beta} \in \mathbb{R}^{p} , \beta_0 \in \mathbb{R} }~~~ 
\sum \limits_{i=1}^n \left( 1 - y_i (\M{x}_{i}^T\B\beta + \beta_0  ) \right)_+  +  \| \B\beta \|_{S}~~~\sbt~~~ \beta_{j} = 0, j \notin {\mathcal J}.
$$
The above can be expressed as an LP similar to Problem~\eqref{slope-primal} but with fewer columns
\begin{equation}\label{slope-restricted-dual}
\begin{myarray}[1.1]{l c c c}
\boxed{\mathcal{M}_S \left(\mathcal{C}^{\mathcal{J}} , \mathcal{J} \right)}~~  & \min \limits_{ \substack{ \B{\xi} \in \mathbb{R}^n,  \ \beta_0, \ \eta  \in \mathbb{R},\\ 
		\  \B{\beta}^+_{\mathcal{J}} ,\  \B{\beta}^-_{\mathcal{J}} \in \mathbb{R}^{| \mathcal{J}| } } }  &  \sum \limits_{i=1}^n \xi_i  + \eta &\\
&\sbt & \;\; \ \xi_i +  \sum \limits_{j \in \mathcal{J} }  y_i x_{ij} \beta_j^+ - \sum \limits_{j \in \mathcal{J} }  y_i x_{ij} \beta_j^-  + y_i \beta_0 \ge 1 & \;\;\;\; i \in [n]\\
& & \;\;\;\; (\B{\beta}^+_{\mathcal{J}},\  \B{\beta}^-_{\mathcal{J}} , \ \eta) \in \mathcal{C}^{ \mathcal{J} } \\
& & \;\;\;\; \ \B{\xi} \ge 0, \ \B{\beta}^+_{\mathcal{J}} \ge 0, \ \B{\beta}^-_{\mathcal{J}} \ge 0
\end{myarray}
\end{equation}
where, $\B\beta_{\mathcal J}$ is a sub-vector of $\B\beta$ restricted to $\mathcal J$ and $\mathcal{C}^{ \mathcal{J} }$ is the adaption of~\eqref{C-definition} restricted to $\B\beta_{\mathcal J}$:
\begin{equation}\label{CJ-definition}
\mathcal{C}^{ \mathcal{J} } := \left\{  \left(\B{\beta}^+_{\mathcal{J}},\  \B{\beta}^-_{\mathcal{J}} , \ \eta \right)  ~~\biggr\rvert ~~
 \B{\beta}^+_{\mathcal{J}}, \B{\beta}^-_{\mathcal{J}} \in \mathbb{R}^{| \mathcal{J} |}, ~ \eta \ge \max \limits_{  \B{w}_{\mathcal{J}} \in \mathcal{W}^{  \mathcal{J}  } } \B{w}^T_{\mathcal{J}}  (\B{\beta}^+_{\mathcal{J}} + \B{\beta}^-_{\mathcal{J}}) \right\}
\end{equation}
where, $\B{w}_{\mathcal J} \in \mathbb{R}^{| \mathcal{J} |}$ and $\mathcal{W}^{  \mathcal{J}  }$ is defined as:
$$\mathcal{W}^{\mathcal{J}} := \left\{\B{w}_{\mathcal J}  \ \biggr\rvert \  \exists \psi \in \mathcal{S}_{| \mathcal{J} |}~~\text{s.t.}~~  w_{ \mathcal{J}(j) } = \lambda_{\psi(j)}, \ \forall j \le | \mathcal{J} |  \right\}.$$
 
Since column generation is equivalent to constraint generation on the dual problem, to determine the set of columns to add to $\mathcal{J}$ in Problem \eqref{slope-restricted-dual}, we need the dual formulation of Slope-SVM.

\smallskip

\noindent {\bf Dual formulation for Slope-SVM:} 
We first present a dual~\cite{owl} of the Slope norm:
\begin{equation}\label{dual-form-slope-1}
\max \left\{\B{\beta}^T \B{z} ~\biggr\rvert ~ \B{\beta} \in \mathbb{R}^p, \ \| \B{\beta} \|_S \le 1 \right\} 
= \max_{ k \le p} \left\{\left( \sum_{j=1}^k \lambda_j \right)^{-1} \sum_{j=1}^k  | z_{(j)} | \right\}. \end{equation}
The identity~\eqref{dual-form-slope-1} follows from the observation that the maximum will be attained at an extreme point of the polyhedron $\mathcal{P}_S = 
\{\B{\beta} ~ | \ \| \B{\beta} \|_S \le 1 \} \subset \mathbb{R}^p$. We describe these extreme points.
We fix $k \in [p]$, and a subset $A \subset \left\{1, \ldots, p\right\}$ of size $k$ --- the extreme points of $\mathcal{P}_S$ 
having support $A$ have their nonzero coefficients to be equal, with absolute value $\left(\sum_{j=1}^k \lambda_j \right)^{-1}$. 
Finally,~\eqref{dual-form-slope-1} follows by taking a maximum over all $k \in [p]$.

A dual of Problem \eqref{slope-restricted-dual} is given by:
\begin{equation}\label{slope-dual}
    \begin{aligned}
    \max \limits_{\B{\pi} \in \mathbb{R}^{n}, \M{q} \in  \mathbb{R}^p} &~~~~~ \sum \limits_{i=1}^n \pi_i   \\
   \sbt & \;\;\;\; \  \max\limits_{k=1, \ldots, | \mathcal{J} |} \left\{ \left( \sum \limits_{j=1}^k \lambda_j \right)^{-1} \sum \limits_{j=1}^k  | q_{(j)} |  \right\}  \le 1   \\
   & \;\;\;\; \ q_j = \sum_{i=1}^n  y_i x_{ij} \pi_i, ~~~ j \in [p]  \\
  & \;\;\;\; \  \mathbf{y}^T \B{\pi} = 0    \\
  &  \;\;\;\; \   0\le \pi_i \le 1,~~ i \in [n]. 
    \end{aligned}
\end{equation}
 We now discuss how additional columns can be appended to $\mathcal J$ in Problem \eqref{slope-restricted-dual} to perform column generation. 
  Let $\B{\pi}^* \in \mathbb{R}^n$ be an optimal solution of Problem \eqref{slope-dual}. We compute the associated $\M{q}^*$ and sort its entries such that  $| q^*_{ (1) } | \ge \ldots \ge | q^*_{ ( | \mathcal{J} |)} |$. The first constraint in~\eqref{slope-dual} leads to:
\begin{equation}\label{slope-reduced-cost-bis}
\max_{k=1, \ldots, | \mathcal{J} |} \left\{ \sum_{j=1}^k  | q^*_{ (j) } | -  \sum_{j=1}^k \lambda_j \right\}  \le 0. 
\end{equation}
Now, for each column $j \notin \mathcal{J} $,  we compute its corresponding $q^*_{(j)} $ and insert it into the sorted sequence $| q^*_{ (1) } | \ge \ldots \ge | q^*_{ ( | \mathcal{J} |)} |$. 
This insertion costs at most $O(| \mathcal{J} |)$ flops: 
we update $\mathcal{J} \leftarrow \mathcal{J} \cup \left\{j\right\}$
and denote the sorted entries by:
$| q^*_{\;\; \overline{ (1)} } | \ge \ldots \ge | q^*_{ \;\; \overline{ ( | \mathcal{J} | + 1)} } |$. We add a column $j \notin \mathcal{J}$ to the current model if:
\begin{equation}\label{slope-reduced-cost}
\max_{k=1, \ldots, | \mathcal{J} |+1} \left\{ \sum_{j=1}^k  | q^*_{ \;\; \overline{(j)}} | -  \sum_{j=1}^k \lambda_j \right\}  > \epsilon,
\end{equation}
and this costs $O(| \mathcal{J} |+1)$ flops. Therefore, the total cost of sorting the vector $\M{q}^*$ and scanning through all columns (not in the current model) for negative reduced costs, is of the 
order $\mathcal{O}\left( |\mathcal{J}| \log | \mathcal{J} | + 2(p - |\mathcal{J} |)  |\mathcal{J} |  \right)$.  
This approach can be computationally expensive. 
To this end, we propose an alternative method having a smaller cost with $\mathcal{O}\left( |\mathcal{J} |  \right)$ flops.
Indeed, by combining Equations \eqref{slope-reduced-cost-bis} and \eqref{slope-reduced-cost}, a column $j \notin \mathcal{J}$ will be added to the model if it satisfies: 
\begin{equation}\label{relation-add}
| q_{j} | \ge \lambda_{ | \mathcal{J} | +1 } + \epsilon.
\end{equation}
This shows that the cost of adding a new column for Slope-SVM is the same as that in L1-SVM. The column generation algorithm is summarized below.

\medskip
\medskip
\medskip

\noindent \textsc{Algorithm~6:} \textbf{Column generation for Slope-SVM} \newline
\textbf{Input: } $\mathbf{X}$, $\mathbf{y}$, a sequence of Slope coefficients  $\left\{ \lambda_j \right\}$, a threshold $\epsilon>0$, an initial set of columns $\mathcal{J}$. 
 \newline
\textbf{Output: }A near-optimal solution $\B{\beta}^*$ for the Slope-SVM Problem \eqref{Slope-intro}.
\begin{enumerate}
	\item 
	Repeat Steps 2 to 3 until no column can be added.
	
	\item Solve the model 	$\mathcal{M}_S \left(  \mathcal{C}^{\mathcal{J} }, \mathcal{J} \right)$ in Problem \eqref{slope-primal-restricted} with warm-start (if available).
	
	\item Identify the columns $\mathcal{J}^{\epsilon} \subset \left\{1,\ldots,p \right\} \backslash \mathcal{J}$ that need to be added by using criterion~\eqref{relation-add}. Update  $\mathcal{J} \leftarrow \mathcal{J} \cup \mathcal{J}^{\epsilon}$, and go to Step~2.
\end{enumerate}

\subsection{Pairing column and constraint generation for Slope-SVM}\label{sec:slope-CCG}
We discuss how to combine the column generation (Section~\ref{sec:slope-columns}) and constraint generation (Section~\ref{con-gen-slope}) methods outlined above to solve the Slope SVM problem.

For a set of columns $\mathcal{J}$ and constraints associated with $\B{w}^{(1)}_{\mathcal J}, \ldots, \B{w}^{(t)}_{\mathcal J} \in \mathcal{W}^{\mathcal{J} }$, we consider the following problem
\begin{equation}\label{slope-doubly-restricted-primal}
\begin{myarray}[1.1]{c c c c r}
\boxed{\mathcal{M}_S \left( \mathcal{C}_t^{\mathcal{J}}, \mathcal{J}  \right)}  & \min \limits_{ \substack{\B{\xi} \in \mathbb{R}^n,  \ \beta_0  \in \mathbb{R}, \ \eta \in \mathbb{R} \\ \B{\beta}^+_{\mathcal J},\  \B{\beta}^-_{\mathcal J} \in \mathbb{R}^{| \mathcal{J}|} }}  &  \sum \limits_{i=1}^n \xi_i  + \eta& \\
&\sbt & \;\; \ \xi_i +  \sum \limits_{j \in \mathcal{J} }  y_i x_{ij} \beta_j^+ - \sum \limits_{j \in \mathcal{J} }  y_i x_{ij} \beta_j^-  + y_i \beta_0 \ge 1 & \;\;\; i \in [n]\\
& & \;\;\;\; (\B{\beta}^+_{\mathcal J},\  \B{\beta}^-_{\mathcal J} , \ \eta) \in \mathcal{C}_t^{ \mathcal{J} } & \\
& & \;\;\;\; \ \B{\xi} \ge 0, \ \B{\beta}^+_{\mathcal J} \ge 0, \ \B{\beta}^-_{\mathcal J} \ge 0,  &\\
\end{myarray}
\end{equation}
where, $\mathcal{C}_t^{\mathcal{J}}$ (and $\mathcal{C}^{\mathcal{J}}$) is a restriction of $\mathcal{C}_t$ (and $\mathcal{C}$, respectively) to the columns ${\mathcal J}$. Formally,

$$ \mathcal{C}_t^{\mathcal{J} }  := \left\{  \left(\B{\beta}^+_{\mathcal J},\  \B{\beta}^-_{\mathcal J} , \ \eta \right)  \ \biggr\rvert ~
\B{\beta}^+_{\mathcal J},\  \B{\beta}^-_{\mathcal J} \in \mathbb{R}^{| \mathcal{J} |},~~ \ \eta \ge 
(\B{w}^{(\ell)}_{\mathcal J} )^T (\B{\beta}^+_{\mathcal J} + \B{\beta}^-_{\mathcal J} ), \ \forall \ell \le t \right\} \supset \mathcal{C}^{\mathcal{J} }. $$

We use the method in Section~\ref{con-gen-slope} to refine $\mathcal{C}_t^{\mathcal{J}}$ and the method of 
Section~\ref{sec:slope-columns} to add a set of columns to $\mathcal J$. 
We use criterion~\eqref{relation-add} to select the columns to add. Let $\mathcal{J}^{\epsilon}$ denote these additional columns with  
coordinates $\mathcal{J}^{\epsilon}(k)$ for $k =1, \ldots, | \mathcal{J}^{\epsilon} |$ --- we will also assume that the elements have been sorted by increasing reduced costs.
For notational purposes, we will need to map\footnote{In other words, the existing vectors $\M{w}^{(\ell)}_{ \mathcal{J}}$ are in  $\mathbb{R}^{| \mathcal{J} |}$ and we need to extend them to $\mathbb{R}^{| \mathcal{J} | +  | \mathcal{J}^{\epsilon} |}$. Therefore, we need to define the coordinates corresponding to the new indices $\mathcal{J}^{\epsilon}$.} 
the existing cuts of $\mathcal{W}^{ \mathcal{J} }$ onto $\mathcal{W}^{ \mathcal{J}  \cup  \mathcal{J}^{\epsilon} }$. 
To this end, 
we make the following definition:
\begin{equation}\label{augment-space-plane}
w^{(\ell)}_{m} = \lambda_{ | \mathcal{J} | + k }, \ \forall m \in \mathcal{J}^{\epsilon}, \ \forall \ell \le t.
\end{equation}

We summarize our algorithm below.

\medskip

\noindent \textsc{Algorithm~7:} \textbf{Column-and-constraint generation for Slope-SVM} \newline
\textbf{Input: } $\mathbf{X}$, $\mathbf{y}$, a sequence of Slope coefficients  $\left\{ \lambda_j \right\}$, a convergence threshold $\epsilon \geq 0$. 
Initialization of $\B{\beta}^{*}$ and $\mathcal{J}$ (e.g., using the first order method in Section \ref{sec: AGD}). Define $\B{w}^{(1)}_{\mathcal J} $ as per~\eqref{adding-cut-slope}.
\newline
\textbf{Output: }A near-optimal solution $\B{\beta}^*$ for the Slope-SVM Problem \eqref{Slope-intro}.
\begin{enumerate}
	\item 
	Repeat Steps 2 to 4 until no cut can be added and $\mathcal{J}$ stabilizes.
	
	\item Solve the model $\mathcal{M}_S \left(  \mathcal{C}_t^{\mathcal{J}}, \mathcal{J} \right) $ in Problem \eqref{slope-doubly-restricted-primal} (with warm-starting enabled).
	
	\item If $\eta < \sum_{j=1}^{| \mathcal{J} |} \lambda_j | \beta^{*}_{(j)}| - \epsilon $, add a new cut $\B{w}^{(t+1)}_{\mathcal J} \in \mathcal{W}^{ \mathcal{J}}$ as in Equation \eqref{adding-cut-slope} and define $\mathcal{C}_{t+1}^{ \mathcal{J} }$.
	
	\item Identify columns $\mathcal{J}^{\epsilon} \subset \left\{1,\ldots,p \right\} \backslash \mathcal{J}$ that need to be added (based on criterion \eqref{relation-add}). 
	Map the cuts $\B{w}^{(1)}_{\mathcal J} , \ldots \B{w}^{(t+1)}_{\mathcal J}$ to  $\mathcal{W}^{ \mathcal{J} \cup \mathcal{J}^{\epsilon} }$ via~\eqref{augment-space-plane}. Update  $\mathcal{J} \leftarrow \mathcal{J} \cup \mathcal{J}^{\epsilon}$ and go to Step~2.
\end{enumerate}

\section{First order methods to obtain a low-accuracy solution}\label{sec: FOM}
The computational performance of the column/constraint generation methods described above (for Problems~\eqref{l1-SVM}, \eqref{group-SVM-intro}, \eqref{Slope-intro}) is found to benefit from a good initialization (e.g, a good estimate of columns for 
column generation). We refer the reader to~\cite{desrosiers2005primer} for additional discussions on the importance of having a good initialization for column generation. Finding a good (and easy to compute) initialization for general problems can be  challenging---in our case, 
we propose to leverage low-accuracy solutions\footnote{As our experiments demonstrate, obtaining high accuracy solutions via first order methods can become prohibitively expensive especially, when compared to the column/constraint algorithms presented here.} available from first order methods~\cite{nesterov2004introductorynew}.

Since all the problems~\eqref{l1-SVM},~\eqref{group-SVM-intro} and~\eqref{Slope-intro}, are nonsmooth, we use Nesterov's smoothing technique~\cite{nesterov-smoothing} to 
smooth the nonsmooth hinge-loss function and use 
proximal gradient descent on the composite version~\cite{nesterov2013gradient} of the problem\footnote{For the Group-SVM problem, we use proximal block coordinate methods instead of proximal gradient methods as they lead to better numerical performance.}. 
 These solutions serve as reasonable (initial) estimates for the sets of columns (constraints) necessary for the column (respectively, constraint) generation methods. 
When the number of samples and/or features become larger, a direct application of the first order methods becomes expensive and we use additional heuristics (e.g., feature screening, sub-sampling) for
scalability (Section~\ref{sec:scalability-methods}).

\subsection{Solving the composite form with Nesterov's smoothing}\label{sec: FOM-smoothing}
Note that for a scalar $u$, we have $\max \{0,u\} = \frac{1}{2}(u + |u|) =  \max_{|w| \le 1} \frac{1}{2}(u + wu)$ and this maximum is achieved when $w=\sign(u)$. Hence, the hinge-loss can be expressed as:
\begin{equation}\label{hinge-loss}
\sum \limits_{i=1}^n \left( z_i \right)_+ =  \max \limits_{\| \B{w} \|_{\infty} \le 1}  \sum \limits_{i=1}^n \frac{1}{2} \left[  z_i + w_i z_i  \right], 
\end{equation}
where $z_i = 1 - y_i (\mathbf{x}_i^T\B{\beta} +\beta_0)$. One can obtain a smooth approximation of~\eqref{hinge-loss} as follows:
\begin{equation}\label{smooth-hinge}
H^\tau(\M{z}) :=  \max \limits_{\|\M{w}\|_{\infty} \le 1}  \sum \nolimits_{i \in [n]} \frac{1}{2} \left[ z_i +  w_i z_i \right] - \frac{\tau}{2} \| \B{w} \|_2^2 
\end{equation}
where, $\tau>0$ is a parameter that controls the amount of smoothness in $H^\tau(\M{z})$ and how well it approximates $H^{0}(\M{z})=\sum_{i=1}^n (z_i)_+$. This is formalized in the following lemma adapted from~\cite{nesterov-smoothing}:
\begin{lemma}\label{lemma-smooth-paper}
	The function  
	$\M{z} \mapsto H^{\tau}(\M{z})$ is an $O(\tau)$-approximation for the hinge loss $H^{0}(\M{z})$ i.e., 
	$H^{0}(\M{z}) \in [H^{\tau}(\M{z}), H^{\tau}(\M{z}) + n \tau/2]$ for all $\M{z}$. Furthermore, $H^{\tau}(\M{z})$ has Lipschitz continuous gradient with parameter $1/{(4\tau)}$, i.e., 
	$ \| \nabla H^{\tau}(\M{z}) - \nabla H^{\tau}(\M{z}')\|_2 \leq  1/{(4\tau)} \| \M{z} - \M{z}' \|_2$ for all $\M{z},\M{z}'$.
\end{lemma}
Let us define: 
$$F^{\tau}(\B{\beta} , \beta_0) =  \max \limits_{\|\M{w}\|_{\infty} \le 1} \left\{ \sum \limits_{i=1}^n \frac{1}{2} \left[  1 - y_i (\mathbf{x}_i^T \B{\beta}  +\beta_0) + w_i(1 - y_i (\mathbf{x}_i^T \B{\beta}  +\beta_0) ) \right] - \frac{\tau}{2} \| \B{w} \|_2^2 \right\}.$$
By Lemma~\ref{lemma-smooth-paper}, it follows that $F^{\tau}(\B{\beta} , \beta_0)$ is an uniform $O(\tau)$-approximation to the hinge-loss function. The gradient of $F^{\tau}$ is given by:
\begin{equation} \label{smoothed-gradient}
\nabla F^{\tau}( \B{\beta}, \ \beta_0 ) = - \frac{1}{2} \sum \limits_{i=1}^{n}(1+w_i^{\tau})y_i \tilde{\mathbf{x}}_i  \in \mathbb{R}^{p+1},
\end{equation}
and $(\B\beta, \beta_0) \mapsto \nabla F^{\tau}( \B{\beta}, \ \beta_0 )$ is Lipschitz continuous
with parameter $C^{\tau} = \sigma_{\max} (\tilde{\M{X}}^T \tilde{\M{X}}) / (4 \tau)$, where $\tilde{\M{X}}_{n \times (p+1)}$ is a matrix with $i$th row $(\M{x}_{i}, 1)$.

We use a proximal gradient method~\cite{FISTA} to the following composite form of the smoothed-hinge-loss SVM problem with regularizer 
$\Omega(\B\beta)$
\begin{equation} \label{smoothed-l1-SVM}
\min \limits_{ \B{\beta} \in \mathbb{R}^p , \beta_0 \in \mathbb{R} } 
F^{\tau}(\B{\beta} , \beta_0)  + \Omega( \B{\beta} ), 
\end{equation}
where, $\Omega( \B{\beta} ) = \lambda \| \B{\beta} \|_1$ for L1-SVM,  $\Omega( \B{\beta} ) = \lambda \sum_{g=1}^G \| \B{\beta}_g \|_{\infty}$ for Group-SVM and $\Omega( \B{\beta} ) = \| \B{\beta} \|_S$ for Slope-SVM. 
For these choices, the proximal/thresholding operators can be computed 
efficiently, as we discuss next.

\subsection {Thresholding operators} \label{sec:thresholding}
For notational convenience we set $\B{\gamma} = (\B{\beta}, \beta^0) \in \mathbb{R}^{p+1}$.  Following~\cite{nesterov2004introductorynew,FISTA} for $L\ge C^{\tau}$, we have  
that $\B\gamma \mapsto Q_L(\B{\gamma};\B{\alpha})$ is an upper bound to $\B\gamma \mapsto F^{\tau}( \B{\gamma} )$, i.e, for all $\B{\alpha}, \B{\gamma}  \in \mathbb{R}^{p+1}$:
\begin{equation}\label{convex-continuous-gradient}
F^{\tau}( \B{\gamma} ) \le Q_L(\B{\gamma};\B{\alpha}) := F^{\tau}(\B{\alpha}) + \nabla F^{\tau}(\B{\alpha})^T(\B{\gamma}-\B{\alpha}) + \frac{L}{2} \| \B{\gamma} - \B{\alpha} \|_2^2.
\end{equation}
The proximal gradient method requires solving the following problem:
\begin{equation} \label{key-algortihm1}
\hat{\B{\gamma}} = \argmin_{ \B{\gamma}} \left\{ Q_L(\B{\gamma};\B{\alpha}) + \Omega( \B{\gamma} )  \right\}
= \argmin_{ \B{\gamma}  }\frac{1}{2} \left\| \B{\gamma} - \left(\B{\alpha} - \frac{1}{L} \nabla F^{\tau}(\B{\alpha})  \right) \right\|_2^2 + \frac{1}{L} \Omega( \B{\gamma} ).
\end{equation}
We denote: $\hat{\B{\gamma}}=(\hat{\B{\beta}}, \hat{\beta^0})$. Note that $\hat{\beta_0}$ is simple to compute and  $\hat{\B{\beta}}$ can be computed via the following thresholding operator (with $\mu>0$):
\begin{equation}\label{thresh-op1}
\mathcal{S}_{\mu \Omega }(\B{\eta}) := \argmin_{ \B{\beta} \in \mathbb{R}^p  }\frac{1}{2} \left\| \B{\beta} - \B{\eta}  \right\|_2^2 + \mu \Omega( \B{\beta} ).
\end{equation}
Computation of the thresholding operator is discussed below for specific choices of $\Omega$.

\paragraph{Thresholding operator when $\Omega(\B\beta) = \|\B\beta \|_{1}$:} 
In this case, $\mathcal{S}_{\mu \Omega }(\B{\eta})$ is available via componentwise soft-thresholding where, the scalar soft-thresholding operator is given by:
$$\argmin_{ u \in {\mathbb R}} ~ \frac{1}{2}(u-c)^2 + \mu  | u| = \sign(c) (|c| - \mu)_+.$$

\paragraph{Thresholding operator when $\Omega(\B\beta) = \sum_{g \in [G]} \|\B\beta_{g} \|_{\infty}$:}
We first consider the projection operator that projects onto an L1-ball with radius $\mu >0$
\begin{align} \label{key-algortihm1-group-bis}
\tilde{\mathcal{S}}_{ \tfrac1\mu \| \cdot \|_{1}  }  (\B{\eta}) := \argmin_{ \B{\beta}}~ \frac{1}{2}\left\| \B{\beta} - \B{\eta}  \right\|_{2}^2  \ \ \sbt \;\ \ \frac1\mu \| \B{\beta} \|_{1} \le 1.
\end{align}
From standard results pertaining to the Moreau decomposition~\cite{moreau1962fonctions} (see also~\cite{bach2011convex}) we have:
\begin{equation} \label{l_inf_prox}
\mathcal{S}_{\mu \| . \|_{\infty}} (\B\eta) + \tilde{\mathcal{S}}_{\tfrac1\mu \| \cdot \|_{1}}(\B\eta) = \B\eta
\end{equation}
for any  $\B\eta$. 
Note that $\tilde{\mathcal{S}}_{\tfrac1\mu \| \cdot \|_{1}}(\B\eta)$ 
can be computed via a simple sorting operation~\cite{L1-proj-github,L1-proj}, leading to a solution for $\mathcal{S}_{\mu \| . \|_{\infty}} (\B\eta)$.
This observation can be used to solve Problem \eqref{thresh-op1} with the L1/$L_\infty$ Group regularizer by noticing that the problem separates across the $G$ groups.

\paragraph{Thresholding operator when $\Omega(\B\beta) = \sum_{i \in [p]} \lambda_{i} |\beta_{(i)} |$:}
For the Slope regularizer, Problem~\eqref{thresh-op1} reduces to the following optimization problem:
\begin{equation}\label{slope-thresh-1}
\min_{\B\beta}~~\frac12 \| \B\beta - \B\eta \|_{2}^2 +  \mu \sum_{i} \lambda_{i} | \beta_{(i)} |. 
\end{equation}
As noted by~\cite{slope-proximal}, at an optimal solution to Problem~\eqref{slope-thresh-1}, the signs of $\beta_{j}$ and $\eta_j$ are the same.
In addition, since $\lambda_{i}$'s are decreasing, a solution to Problem~\eqref{slope-thresh-1} can be found by solving the following close relative to the isotonic regression problem~\cite{robertson1988order}
\begin{equation}\label{slope-thresholded}
\min \limits_{\M{u}} ~~ \frac{1}{2}  \left\| \M{u} - \tilde{\B{\eta}} \right\|_2^2 +  \sum \limits_{j=1}^p \mu \lambda_j u_j  
~~~~~\sbt~~~~ u_1 \ge \ldots \ge u_p \ge 0
\end{equation}
where, $\tilde{\B{\eta}}$ is a decreasing re-arrangement of the absolute values of $\B\eta$, with $\tilde{\eta}_{i} \geq \tilde{\eta}_{i+1}$ for all $i$.
If $\hat{\M{u}}$ is a solution to Problem~\eqref{slope-thresholded}---then its $i$th coordinate $\hat{u}_{i}$
corresponds to $|\hat{\beta}_{(i)}|$ where, $\hat{\B\beta}$ is an optimal solution of Problem~\eqref{slope-thresh-1}.

\subsection{Deterministic first order algorithms}\label{sec: AGD}

\noindent {\bf (Accelerated) Proximal gradient descent:} 
Let us denote the mapping~\eqref{key-algortihm1} $\B\alpha \mapsto \hat{\B\gamma}$ by the operator: 
$ \hat{\B\gamma} :=  \Theta(\B\alpha).$
The basic version of the proximal gradient descent algorithm performs the updates: $\B\alpha_{T+1} = \Theta(\B\alpha_{T})$ (for $T \geq 1$) after starting with $\B\alpha_1 = (\B\beta_{1}, \beta^0_1)$.
The accelerated gradient descent algorithm~\cite{FISTA}, which enjoys a faster convergence rate performs updates with a minor modification.
It starts with ${\B\alpha}_{1} = \tilde{\B\alpha}_{0}$, $q_{1} =1$ and then performs the updates:
$ \tilde{\B\alpha}_{T+1} = \Theta({\B\alpha}_{T})$ where, 
${\B\alpha}_{T+1} = \tilde{\B\alpha}_{T} + \frac{q_{T} -1}{q_{T+1}} (\tilde{\B\alpha}_{T} - \tilde{\B\alpha}_{T-1} )$
and $q_{T+1} = (1 + \sqrt{1 + 4 q_T ^2})/{2}$. This algorithm requires $O(1/{\epsilon})$ iterations to reach an $\epsilon$-optimal solution for the 
original problem (with the hinge-loss). We perform these updates till some tolerance criterion is satisfied, for example, $\|\B\alpha_{T+1} - \B\alpha_{T} \| \leq \eta$ for some tolerance level $\eta >0$. 
In most of our examples (cf Section~\ref{sec: computation}), we set a generous (or loose) tolerance of $\eta=10^{-3}$ or run the algorithm with a 
limit on the total number of iterations (usually a couple of hundred)\footnote{This choice is user-dependent---there is a tradeoff between computation time and the quality of solution. We recommend using a low accuracy solution as its purpose is to serve as an initialization for the column/constraint generation method}.

\medskip

\noindent {\bf Block Coordinate Descent (CD) for the Group-SVM problem:}  We describe a cyclical proximal block coordinate (CD) descent algorithm~\cite{wright2015coordinate} for the smooth hinge-loss function with the group regularizer. For the group-SVM experiments considered in Section~\ref{sec: computation}, the block CD approach was found to exhibit superior numerical performance compared to a full gradient-based procedure.
We note that~\cite{blockCD-lasso} explore block CD like algorithms for a different class of group Lasso type problems\footnote{The authors in~\cite{blockCD-lasso} consider a different class of problems than those studied here; and use exact minimization for every block (they use a squared error loss function).} and our approaches differ.

We perform a proximal gradient step on the $g$th group of coefficients (with all other blocks and $\beta_0$ held fixed) via:
\begin{equation} \label{group-g-update}
\B{\beta}_g^{t+1}\in~\argmin_{ \B{\beta}_g  } \frac{1}{2}\left\| \B{\beta}_g- \B{\beta}_g^t - \frac{1}{C^{\tau}_g}  \left\{ \nabla F^{\tau}(\B{\beta}_1^{t+1},\ldots \B{\beta}_{g-1}^{t+1}, \B{\beta}_g^t, \ldots \B{\beta}_G^t, \beta_0^{t} )  \right\}_{\mathcal{I}_g } \right\|_2^2 + \frac{\lambda}{C^{\tau}_g} \| \B{\beta}_g \|_{\infty},
\end{equation}
where $\left\{ \nabla F^{\tau} (\cdot) \right\}_{\mathcal{I}_g }$ denotes the gradient restricted to the coordinates $\mathcal{I}_g$ and $C^{\tau}_g$ is its associated Lipschitz constant: $C^{\tau}_g = \sigma_{\max} (\mathbf{X}^T_{\mathcal{I}_g} \mathbf{X}_{\mathcal{I}_g} ) / 4 \tau$. We cyclically update the coefficients across each group $g\in[G]$ and then update $\beta^0$. This continues till some convergence criterion is met.

Computational savings are possible for this block CD algorithm by 
a careful accounting of floating point operations (flops). As one moves from one group to the next, the whole gradient can be updated easily.
To this end, note that the gradient $\nabla F^{\tau}( \B{\beta}, \beta_0 )$ restricted to block $g$ is given by:
$$\left\{  \nabla F^{\tau}( \B{\beta}, \beta_0 ) \right\}_{\mathcal{I}_g } = - \frac{1}{2} \mathbf{X}_{\mathcal{I}_g}^T \left\{ \mathbf{y}\circ(1+
\M{w}^{\tau} )\right\},$$
where `$\circ$' denotes element-wise multiplication. If $\M{w}^{\tau}$ is known, the above computation requires $n| \mathcal{I}_g |$ flops. Recall that $\M{w}^{\tau}$ depends upon $\B\beta$ via: 
$w^{\tau}_i = \min\left( 1, \frac{1}{2\tau} | z_i | \right) \sign(z_i)$ where $z_i = 1 - y_i (\mathbf{x}_i^T\B{\beta}  + \beta_0), \ \forall i$. If $\B\beta$ changes from $\B{\beta}^{\text{old}}$ to $\B{\beta}^{\text{new}}$, then $\M{w}^\tau$ changes  via  an update in $\mathbf{X} \B{\beta}$ --- this change can be efficiently computed by noting that:
$\mathbf{X} \B{\beta}^{\text{new}} =  \sum_{g\in[G]} \mathbf{X}_{\mathcal{I}_g} \B{\beta}^{\text{new}}_g =  \mathbf{X} \B{\beta}^{\text{old}} + \mathbf{X}_{\mathcal{I}_g} \Delta \B{\beta}_g$
where, $\Delta \B{\beta}_g = \B{\beta}_g^{\text{new}} - \B{\beta}_g^{\text{old}}$ is a change that is only restricted to block $g$.
Hence updating $\M{w}^{\tau}$ also requires $n| \mathcal{I}_g |$ operations. 
The above suggests that one sweep of block CD across all the coordinates has a cost similar to that of computing a full gradient. 
In addition, techniques like active set updates and warm-start continuation~\cite{friedman8} can lead to improved computational performance for CD, in practice.

\subsection{Scalability heuristics for large problem instances}\label{sec:scalability-methods}
When $n$ and/or $p$ becomes large, the first order algorithms discussed above become expensive. Recall that the goal of the first order methods is to get a low-accuracy solution for the SVM problem and in particular, an estimate of the initial columns and/or constraints for the column/constraint generation algorithms. Hence, for scalability purposes, we use principled heuristics as a wrapper around the first order methods, discussed above.

\subsubsection{Correlation screening when $p$ is large and $n$ is small}\label{corr-screen-1}
When $p\gg n$, we use a feature screening method inspired by correlation screening~\cite{tibshirani2012strong}, to restrict the number of features (or groups in the case of Group-SVM). We apply the 
first order methods on this reduced set of features. Usually, for L1-SVM and Slope-SVM, we select for example, the top $10n$ columns with highest absolute inner product (note that the features are standardized to have unit L2-norm) 
with the output. For the Group-SVM problem: for each group,
we compute the inner products between every feature within this group and the response, and take their L1-norm. We then sort these numbers and take the top $n$ groups.

\subsubsection{A subsampling heuristic when $n$ is large and $p$ is small} \label{sec: heuristic-aggregation}
The methods described in Section~\ref{sec: AGD} become expensive due to gradient computations when $n$ becomes large. 
When $n$ is large but $p$ is small, we use a subsampling method inspired by~\cite{lee2017communication}. 
To get an approximate solution 
to Problem~\eqref{l1-SVM} we apply the algorithm in Section~\ref{sec: AGD} on a subsample $(y_i, \M{x}_{i}), i \in {\mathcal A}$ with sample-indices
${\mathcal A} \subset [n]$. We (approximately) solve Problem~\eqref{l1-SVM} with $\lambda \leftarrow \tfrac{|\mathcal A|}{n}\lambda$ (to adjust the dependence of $\lambda$ on the sample size) by using the algorithms in Section~\ref{sec: AGD}.
Let the solution obtained be given by $\hat{\B\beta}({\mathcal A})$. We 
obtain $\hat{\B\beta}({\mathcal A}_{j})$ for different subsamples 
${\mathcal A}_{j}$, $j \in [Q]$ and average the estimators\footnote{We note that the estimates $\hat{\B\beta}({\mathcal A_j})$ can all be computed in parallel.} $\bar{\B\beta}_{Q} = \frac{1}{Q} \sum_{j \in [Q]} \hat{\B\beta}({\mathcal A_j})$. We maintain a counter for $Q$, and stop as soon as the average stabilizes\footnote{We note that when $n$ is large (but $p$ is small), basic principles of statistical inference~\cite{van2000empirical} suggest that the estimator $\hat{\B\beta}({\mathcal A_j})$ will serve as a proxy of a minimizer of Problem~\eqref{l1-SVM} --- we average the estimators $\hat{\B\beta}({\mathcal A_j})$'s to reduce variance.}, i.e.,  
$\|\bar{\B\beta}_{Q} -\bar{\B\beta}_{Q-1} \| \leq \mu_{\text{Tol}}$ for some tolerance threshold $\mu_{\text{Tol}}$.
The estimate $\bar{\B\beta}_{Q}$ is used to obtain the violated constraints for the SVM problem and serves to initialize the constraint generation method.

\subsubsection{A subsampling heuristic when both $n$ and $p$ are large}
\label{sec: FOM-CCG}
For large $p$ (small $n$) and large $n$ (small $p$) problems, we basically apply a combination of the ideas described above in Sections~\ref{corr-screen-1} and~\ref{sec: heuristic-aggregation}.
More specifically, we choose a subsample ${\mathcal A}_{j}$ and for this subsample, we use correlation screening to reduce the number of 
features and obtain an estimator $\hat{\B\beta}({\mathcal A_j})$. We then average these estimators (across ${\mathcal A_j}$s) to obtain $\bar{\B\beta}_{Q}$. If the support of $\bar{\B\beta}_{Q}$ is too large, we sort the absolute values of the coefficients and retain the top few hundred coefficients (in absolute value) to initialize the column generation method.  The estimator $\bar{\B\beta}_{Q}$ is used to identify the samples for which the hinge-loss is nonzero --- these indices are used to initialize the 
constraint generation method.

\section{Experiments}\label{sec: computation}

We demonstrate the performance of our different methods on synthetic and real datasets.  
All computations are performed in Python 3.6 on the MIT Engaging Cluster with 1 CPU and 16GB of RAM. We use Gurobi 9.0.2~\cite{gurobi} in our experiments involving Gurobi's LP solver.
Sections~\ref{sec:results-L1-SVM},~\ref{sec-results-group-SVM} and~\ref{sec-results-slope-SVM} present results for the L1-SVM, 
Group-SVM and Slope-SVM problems (respectively).

\subsection{Computational results for L1-SVM}\label{sec:results-L1-SVM}
We present herein our computational experience with regard to the L1-SVM problem. 

\smallskip
\noindent {\textbf{Data Generation:}} We consider $n$ samples from a multivariate Gaussian distribution with 
covariance matrix $\B\Sigma =((\sigma_{ij}))$ with $\sigma_{ij} = \rho$ if $i \ne j$ and $\sigma_{ij}=1$ otherwise.
Half of the samples are from the $+1$ class and have mean $\B{\mu}_+ = (\mathbf{1}_{k_0}, \ \mathbf{0}_{p-k_0} )$. The other half are from the $-1$ class and have mean $\B\mu_- = - \B\mu_+$. We standardize the columns of $\M{X}$ to have unit L2-norm.
In the following results with synthetic datasets, unless otherwise mentioned, we take $ \rho=0.1 $, and $k_0=10$. 

\smallskip
\noindent {\textbf{Measure of accuracy:}}
For an algorithm `$\text{Alg}$', a regularization parameter $\lambda$, 
we let $f^{\text{Alg}}_{\lambda}$ be the objective value obtained by `Alg' for the unconstrained problem~\eqref{l1-SVM}. We let 
$f^{*}_{\lambda}$ be an estimate of the optimal objective value, set to be the lowest objective value among all methods. 
The (relative) accuracy in terms of objective value of a method
`$\text{Alg}$' is given by:
$$\text{ARA} = (f^{\text{Alg}}_{\lambda} - f^{*}_{\lambda})/f^{*}_{\lambda}$$ 
where ARA depends upon $\lambda$ and `Alg'. The ARA value is averaged across $R$-many replications (hence, the shorthand averaged relative accuracy: ARA).

\subsubsection{Synthetic datasets for large $p$ and small $n$}\label{sec: simulations-plarge}

\paragraph{Different initializations for column generation:} We first study the role of different initialization schemes in column generation (denoted by the shorthand CLG below) for solving the L1-SVM problem: we consider a \emph{fixed} value of the regularization parameter, set to
$\lambda = 0.01 \lambda_{\max}$; and
compare the following three schemes:
\begin{enumerate}
	\item[({\bf a})] \texttt{RP-CLG}: We compute a solution to Problem~\eqref{l1-SVM} at the desired value of $\lambda$, using a regularization path (RP) (aka continuation) approach. We compute solutions on a grid of 7 regularization parameter values
	in the range $[\tfrac12\lambda_{\max}, \lambda]$ using column generation (CLG) for every value of the regularization parameter. 
	\item[({\bf b})] \texttt{FO-CLG}: This is the column generation method initialized with a first order (FO) method (cf Section \ref{sec: AGD}) with smoothing parameter $\tau = 0.2$. We use a termination criterion of $\eta=10^{-3}$ or a maximum number of $T_{\max}=200$ iterations for the first order method. 
	We use correlation screening to retain the top $10n$ features before applying the first order method. 
	The time displayed includes the time taken to run the first order method. For reference, we report the time taken to run column generation excluding the time of the first order method: ``\texttt{CLG wo FO}''.
	\item[({\bf c})] \texttt{Cor. screening}: This initializes the column generation method by using correlation screening to retain the top $50$ features. 
\end{enumerate}

\begin{figure}[h!]
		\centering
		\begin{tabular}{cc}
		Runtime vs $p$ & Accuracy (ARA) vs $p$ \\ 
		\includegraphics[width=0.5\textwidth,trim = .3cm .5cm .3cm .9cm,  clip = true]{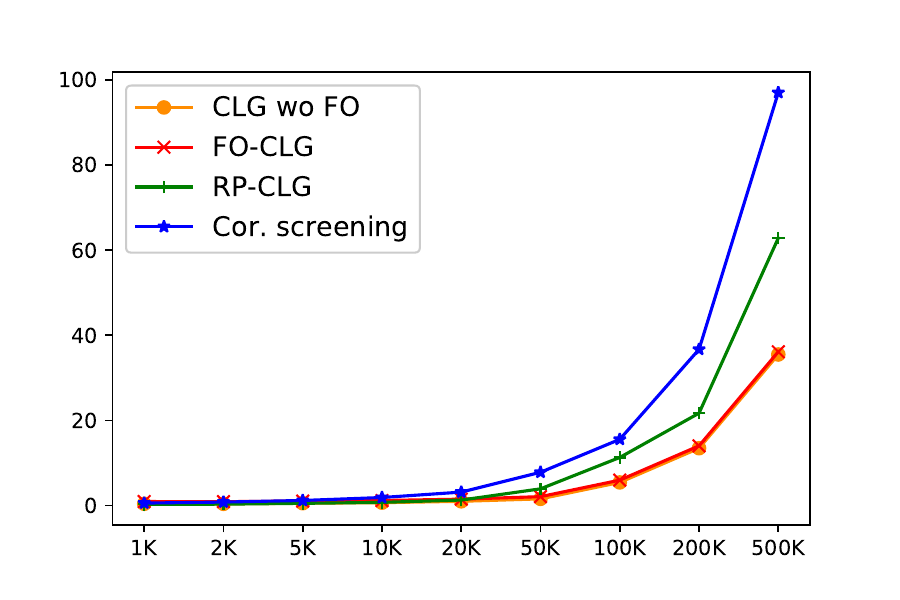}&
		\includegraphics[width=0.5\textwidth,trim = .3cm .5cm .3cm .9cm,  clip = true]{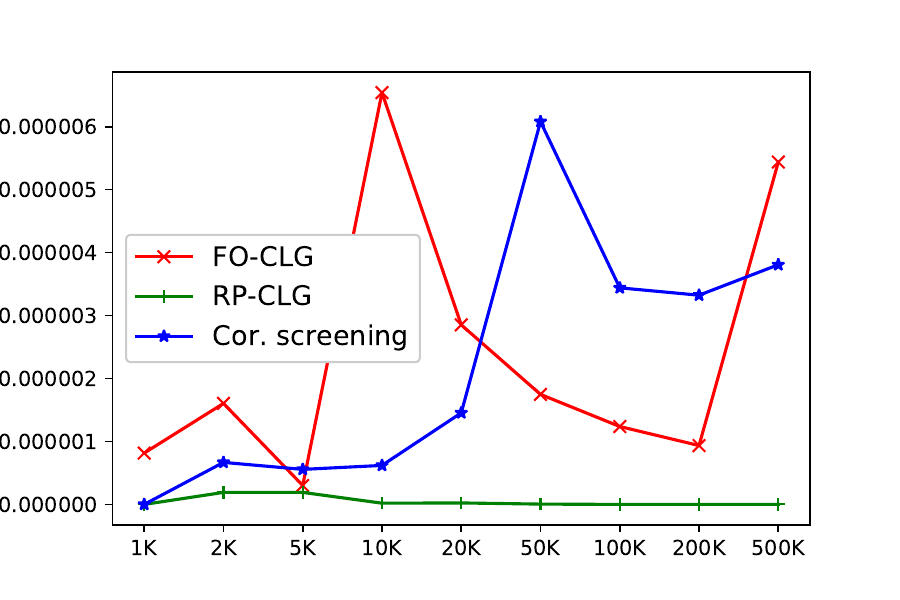} \\
		\# Features $(p)$ & \# Features $(p)$ 
				\end{tabular}
		\caption{ Comparison of different initialization schemes for the column generation method for the L1-SVM LP. Left panel shows runtime (s) versus $p$. Right panel shows the corresponding optimization accuracy (ARA) versus $p$ up to $500,000$ (here, $n = 100$).}
	\label{fig:self-compare}
\end{figure}

The comparative timings between \texttt{FO-CLG} and \texttt{Cor.~screening} show the effectiveness of 
using a first order method to initialize the column generation method. 
Method (a) computes a regularization path (using column generation) to arrive at the desired value of $\lambda$ --- it does not use any first order method like (b) --- thus any timing difference between (a) and (b) is due to the role played by the first order methods for warm-starting.  

Figure \ref{fig:self-compare} shows the results 
for synthetic datasets with $n=100, \ k_0=10, \ \rho=0.1$ and different values of $p$ in the range $1000$ to $500,000$ (results averaged across $R=10$ replications). In this figure, for the column generation methods, we consider a reduced cost threshold of $\epsilon = 0.001$, and set the maximum number of columns to be added in each iteration as 1000.  
The left panel in Figure \ref{fig:self-compare} presents the run times and the right figure presents the ARA of different methods. 
As $p$ increases, the run time for column generation when intitialized with  correlation screening, increases. Column generation is found to benefit the most when initialized with the first order method (denoted by \texttt{FO-CLG}). The runtime of the first order method is negligible compared to the time taken by column generation, as seen from the nearly overlapping profiles of \texttt{FO-CLG} and \texttt{CLG wo FO}. The accuracies of the different procedures (a)--(c) are all quite high  ARA$\sim 10^{-6}$.

\smallskip
\noindent {\textbf{Comparison with benchmarks:}}
We compare the performance of two column-generation methods \texttt{RP-CLG} and \texttt{FO-CLG} with the following benchmarks:

\begin{itemize}
	\item[\textbf{(d)}]\texttt{PSM}: This is a state-of-the-art algorithm~\cite{pang2017parametric} which is a parametric simplex based solver. 
	We use the software made available by~\citet{pang2017parametric} with default parameter-settings. 
	
	\item[\textbf{(e)}]\texttt{FOM}: This runs our first order method, denoted by FOM (we use accelerated gradient descent) with $\tau = 0.02$ for a maximum of 
	$T=100,000$ iterations. We terminate the algorithm if the maximum iteration limit is reached or the change in $\B\beta$ (in $\ell_2$ norm) in the past two iterations is less than $10^{-3}$.
	
	\item[\textbf{(f)}] \texttt{SGD}: This runs a stochastic sub-gradient algorithm on the L1-SVM Problem \eqref{l1-SVM} using Python \texttt{scikit-learn} package implementation (\texttt{SGDClassifier}) with fixed number of $10,000$ iterations. The learning rate is set to the ``optimal" parameter.

	\item[\textbf{(g)}] \texttt{SCS}: 
	This is the Splitting Conic Solver \cite{scs} version 2.1.2 with default parameter setting. The solver is called through \texttt{CVXPY} \cite{diamond2016cvxpy}. This solver is a variant of the ADMM method~\cite{boyd2011distributed}.

	\item[\textbf{(h)}] \texttt{Gurobi}: 
	This is the LP solver of Gurobi~\cite{gurobi} with default setting for solving the full L1-SVM LP. The solver is called through \texttt{CVXPY}~\cite{diamond2016cvxpy}.

\end{itemize}

All the benchmarks used above optimize for the L1-SVM objective function (note that \texttt{FOM} considers a smooth approximation of the hinge-loss).

\begin{table}[h!]
	\centering
\resizebox{\textwidth}{!}{
\begin{tabular}{lc cc cc cc c} 
\multicolumn{7}{c}{$\lam = 0.05 \lam_{\max}$} \\
\cmidrule{1-7}  & \multicolumn{2}{c}{$n=100, p=10K$} & \multicolumn{2}{c}{$n=300, p=10K$} & \multicolumn{2}{c}{$n=100, p=50K$}  \\ \hline
Method          & Time (s)  & ARA & Time (s) & ARA  & Time (s)  & ARA   \\ \hline
\texttt{PSM} & 15.3(1.23) & 6.9e-12(1.43e-12) & 193.8(24.94) & 1.8e-11(5.46e-12) & 178.1(10.05) & 2.0e-11(8.45e-12)  \\ 
\texttt{SGD} & 34.4(0.36) & 9.5e-02(9.98e-03) & 108.8(0.45) & 6.1e-02(5.22e-03) & 172.0(0.87) & 1.6e-01(3.01e-02)  \\ 
\texttt{SCS} & 76.3(3.68) & 1.4e-03(4.47e-04) & 186.9(2.88) & 1.2e-03(1.02e-04) & 456.5(7.19) & 7.3e-01(6.40e-01)  \\ 
\texttt{FOM} & 12.4(0.90) & 2.0e-02(4.31e-04) & 21.4(1.20) & 9.1e-03(2.53e-04) & 92.0(4.65) & 3.0e-02(3.77e-04)  \\ 
\texttt{Gurobi} & 34.5(1.68) & 0.0e+00(0.00e+00) & 93.1(8.66) & 0.0e+00(0.00e+00) & 250.0(8.89) & 0.0e+00(0.00e+00)  \\ 
\texttt{RP-CLG} & 0.4(0.03) & 5.5e-06(3.81e-06) & 2.0(0.14) & 1.0e-06(6.43e-07) & 3.3(0.56) & 8.5e-06(5.79e-06)  \\ 
\texttt{FO-CLG} & 0.6(0.07) & 6.7e-06(3.21e-06) & 2.0(0.14) & 7.6e-06(4.84e-06) & 1.4(0.12) & 2.0e-05(8.16e-06)  \\ \hline 
 $\|\hat{\B{\beta}}\|_0$ \& $| \cJ |$   & \multicolumn{2}{c}{55.0 ~\&~ 265.8}         & \multicolumn{2}{c}{92.8 ~\&~ 218.2}     & \multicolumn{2}{c}{65.6 ~\&~ 257.8}      \\ \hline 
 && && && \\
\end{tabular} } \\
\resizebox{\textwidth}{!}{
\begin{tabular}{lc cc cc cc c} 
 \multicolumn{7}{c}{ $\lambda=0.2 \lam_{\max}$ } \\ \hline
& \multicolumn{2}{c}{$n=100, p=10K$} & \multicolumn{2}{c}{$n=300, p=10K$} & \multicolumn{2}{c}{$n=100, p=50K$}  \\ \hline
Method          & Time (s)  & ARA & Time (s) & ARA  & Time (s)  & ARA   \\ \hline
\texttt{PSM} & 9.5(0.61) & 3.9e-13(1.33e-13) & 87.9(13.15) & 8.9e-13(3.69e-13) & 153.0(10.03) & 2.9e-12(1.41e-12)  \\ 
\texttt{SGD} & 51.9(0.47) & 1.1e-02(1.89e-03) & 162.6(0.44) & 5.3e-03(6.88e-04) & 260.4(1.64) & 1.1e-02(2.29e-03)  \\ 
\texttt{SCS} & 84.4(2.63) & 3.0e-03(7.26e-04) & 201.6(22.27) & 2.3e-03(9.70e-04) & 420.7(10.80) & 4.2e-02(6.05e-03)  \\ 
\texttt{FOM} & 8.2(0.35) & 3.7e-03(3.01e-04) & 15.5(0.79) & 1.0e-03(7.59e-05) & 125.6(0.13) & 6.4e-03(5.38e-04)  \\ 
\texttt{Gurobi} & 26.2(0.89) & 1.6e-16(6.61e-17) & 53.3(8.16) & 1.2e-16(4.83e-17) & 264.9(15.96) & 0.0e+00(0.00e+00)  \\ 
\texttt{RP-CLG} & 0.5(0.10) & 1.4e-07(9.93e-08) & 0.4(0.08) & 1.8e-08(1.65e-08) & 4.8(0.20) & 9.1e-07(8.18e-07)  \\ 
\texttt{FO-CLG} & 0.3(0.00) & 2.2e-08(1.92e-08) & 0.9(0.03) & 1.9e-16(7.34e-17) & 0.9(0.04) & 6.5e-07(4.57e-07)  \\ \hline 
 $\|\hat{\B{\beta}}\|_0$ \& $| \cJ |$   & \multicolumn{2}{c}{36.8 ~\&~ 73.6}         & \multicolumn{2}{c}{30.8 ~\&~ 42.8}     & \multicolumn{2}{c}{53.0 ~\&~ 144.8}      \\ \hline 
\end{tabular}}
	\caption{ {\bf L1-SVM, Synthetic dataset ($p \gg n$)}: Training time (s) for L1-SVM: our proposed column generation method versus various benchmarks on synthetic datasets. We show two different values of $\lambda$: $\lam = 0.05 \lam_{\max}$ [top table] and $\lambda = 0.2\lambda_{\max}$ [bottom table]. Each table presents results for different values of $(n,p)$ with $p \gg n$. For each table, the last row presents the number of nonzeros in $\hat{\B\beta}$ (an optimal solution to the L1-SVM problem)  and the number of columns, $|\cJ|$, in the restricted problem, upon termination of \texttt{FO-CLG}. Results are averaged over 5 replications; and the numbers with parenthesis denote standard errors.} \label{table:simulated-p-large2}
\end{table}

Table~\ref{table:simulated-p-large2} presents the results for different methods for
$\lam = 0.05 \lam_{\max}$ [top panel] and $\lam = 0.2 \lam_{\max}$ [bottom panel]. Here we run \texttt{RP-CLG} and \texttt{FO-CLG} with reduced cost tolerance $\ep = 0.01$. For each combination $(n,p)=(100, 10K)$, $(n,p)=(300, 10K)$ and $(n,p)=(100, 50K)$, and for each method considered,
we show the runtime and associated ARA (results are averaged over 5 replications, and numbers within parenthesis denote standard errors). For the instance in Table \ref{table:simulated-p-large2} with $n=100$ and $p=50K$, we run FOM for 2000 iterations---in this instance, stopping the algorithm with a tolerance $10^{-3}$ leads to a low-quality solution. 

 In addition, to give an idea of the sparsity level of the solution (for the $\lambda$-values chosen), we present the support size of an optimal solution 
$\hat{\B{\beta}}$ (computed by \texttt{Gurobi})  and the number of columns $|\cJ|$ in the restricted problem, upon termination of \texttt{FO-CLG}. From Table~\ref{table:simulated-p-large2}, it can be seen that our proposed methods (\texttt{RP-CLG} and \texttt{FO-CLG}) outperform all the benchmark methods in runtime by a factor of $30X \sim 500X$. In terms of solution accuracy, our column generation methods (reaching an ARA$\sim 10^{-5}$ or smaller)
are comparable to that of \texttt{Gurobi}, \texttt{PSM}.
The operator splitting method \texttt{SCS} leads to solutions of low-accuracy --- their ARA $\sim 10^{-3}$ for $(n,p) = (100, 10K)$ and $(300, 10K)$, but the ARA is slightly larger $\sim 10^{-1}$--$10^{-2}$ for $(n,p) = (100, 50K)$. \texttt{SGD} and \texttt{FOM} also lead to low-accuracy solutions; with \texttt{FOM} leading to somewhat better performance compared to \texttt{SCS} and \texttt{SGD}.
We note that the poor performance of \texttt{SGD} in Table~\ref{table:simulated-n-large2} should not come as a surprise, as stochastic subgradient methods are not designed for small $n$ and large $p$ settings. In addition, given our earlier discussion, deterministic subgradient methods for nonsmooth problems have a slower convergence compared to Nesterov's smoothing technique --- the \texttt{FOM} presented here is an instance of Nesterov's smoothing technique (see Section~\ref{sec: FOM}).

\smallskip
\noindent {\textbf{Performance of CLG for difference tolerance levels $\ep$:}}
{Table~\ref{table:simulated-p-large2} above presents the results for \texttt{FO-CLG} for $\ep = 0.01$. 
In Figure~\ref{fig:different tol} we study sensitivity to the choice of $\epsilon$---we present the runtime and ARA of \texttt{FO-CLG} under different tolerance values $\ep \in \{0.01, 0.03, 0.1, 0.3, 1\}$ for $\lam = 0.05\lam_{\max}$. 
It can be seen that the ARA of \texttt{FO-CLG} changes with $\ep$---with $\ep=0.01$ generally leading to a solution of high-accuracy.
As $\ep$ decreases, the runtime increases slightly.}

\begin{figure}[h!]
\centering
\begin{tabular}{cc}
		\includegraphics[width=0.48\textwidth,trim = .3cm .6cm .3cm .3cm,  clip = true]{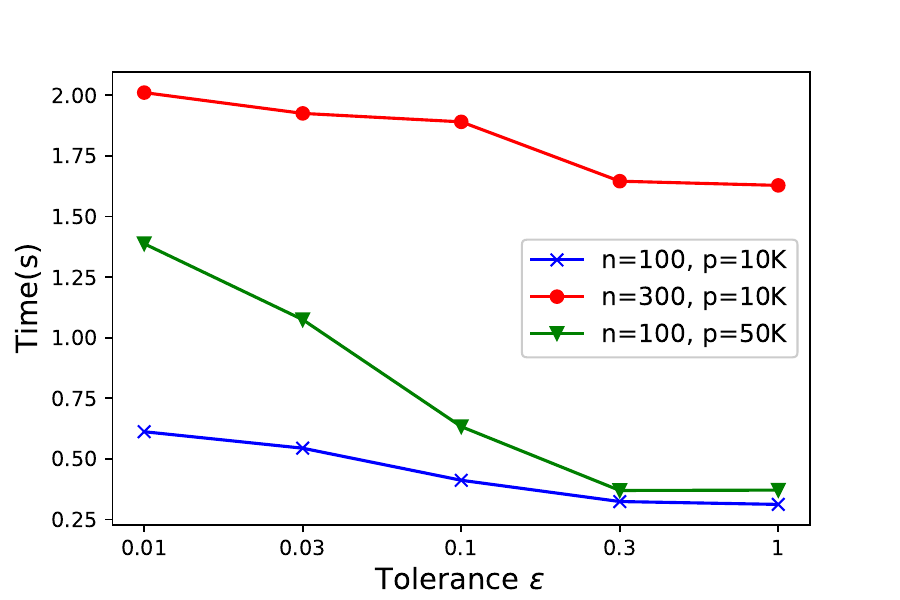}&
		\includegraphics[width=0.48\textwidth,trim = .2cm .6cm .3cm .3cm,  clip = true]{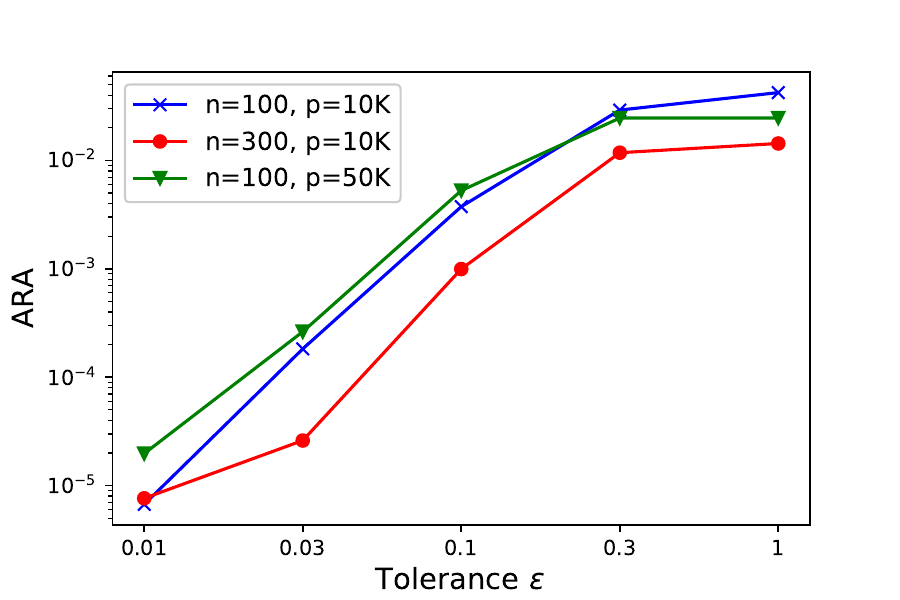} \\
		Tolerance ($\epsilon$) & 	Tolerance ($\epsilon$)
\end{tabular}
\caption{Runtime (s) and ARA of \texttt{FO-CLG} under different tolerance levels $\ep$ for the L1-SVM LP. We consider three different problem-sizes as indicated in the figure legends.}\label{fig:different tol}
\end{figure}

\medskip

\noindent {\textbf{Computing a path of solutions:}}
The results above discuss obtaining a solution to the L1-SVM problem for a fixed value of $\lambda$. Here
we present results for computing L1-SVM solutions for a grid of $\lambda$ values---we focus on comparing the performances of our column generation methods: \texttt{RP-CLG} and \texttt{FO-CLG}.
We fix $n=1000, p =100K,  k_0 = 10, \rho = 0.1$, and a sequence of 50 values of $\lam$: 
\begin{equation}\label{defn-kappa-lambda}
\lam = \kappa \lam_{\max}, ~~
\kappa \in \{0.01, 0.02, \ldots , 0.49, 0.50\}.
\end{equation}
For convenience, let us denote the different values of $\lambda$ by $\lam_1> \lam_2> \cdots > \lam_{m}$ (here, $m = 50$). 
For \texttt{RP-CLG}, we solve the problems in the order $\lam_1$, $\lam_2, ....., \lam_{m-1}, \lam_m$, and use the solution at $\lam_{i}$ as a warm start to compute the solution at $\lam_{i+1}$ (with column generation enabled). 
For \texttt{FO-CLG}, we solve the problems for different $\lam_i$'s \emph{independently}---we use our first order method to initialize the initial set of columns for column generation.  
In the left panel of Figure \ref{fig:path solution}, we present the runtime for both methods at each $\lam_i $. We also present \texttt{CLG wo FO}, which denotes the runtime of column generation (not including the runtime for the first order method) in the implementation of \texttt{FO-CLG}. In the right panel of Figure~\ref{fig:path solution}, we present the support size of the solution $\hat{\B\beta}$ (denoted by ``beta supp" in the figure). As shown in Figure \ref{fig:path solution}, when $\lam$ decreases, the support size of the solution increases, and 
the runtime of both \texttt{RP-CLG} and 
\texttt{FO-CLG} increase. The performance of \texttt{RP-CLG}, which performs warm-starts along the sequence of $\lambda$-values appears to be superior to \texttt{FO-CLG}. As \texttt{FO-CLG} does not use regularization-path continuation, one can compute solutions for the different $\lambda_{i}$-values in parallel, which is not possible with the sequential approach \texttt{RP-CLG}. Note however, in our experiments, 
the L1-SVM solutions for the different $\lambda_{i}$ values are computed sequentially (and not in parallel).
Based on this experiment, we recommend using \texttt{RP-CLG} when one wishes to compute a path of solutions to the L1-SVM problem, in a sequential fashion.

\begin{figure}[h!]

\centering
\begin{tabular}{cc}
		\includegraphics[width=0.45\textwidth,trim = .3cm .6cm .3cm .3cm,  clip = true]{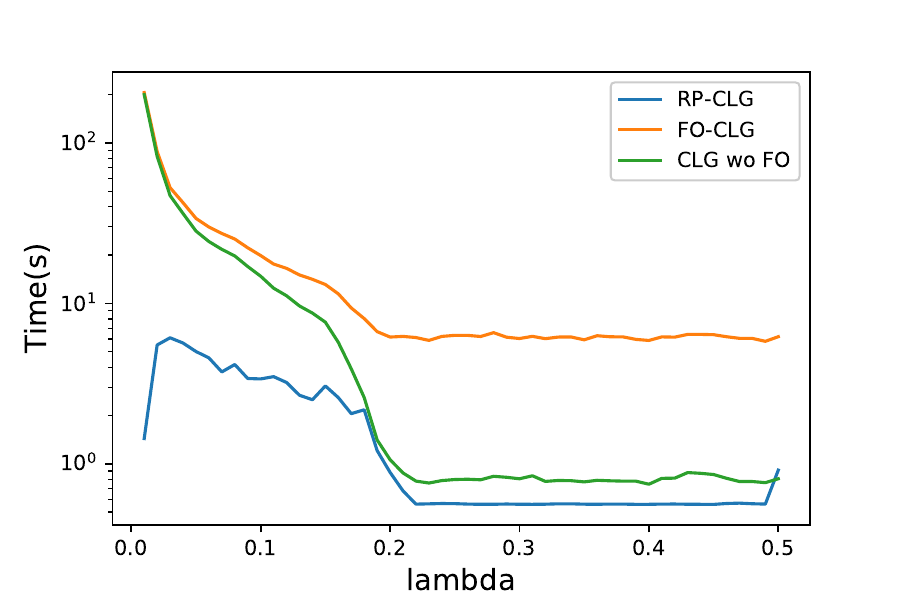}&
		\includegraphics[width=0.45\textwidth,trim = .3cm .6cm .3cm .3cm,  clip = true]{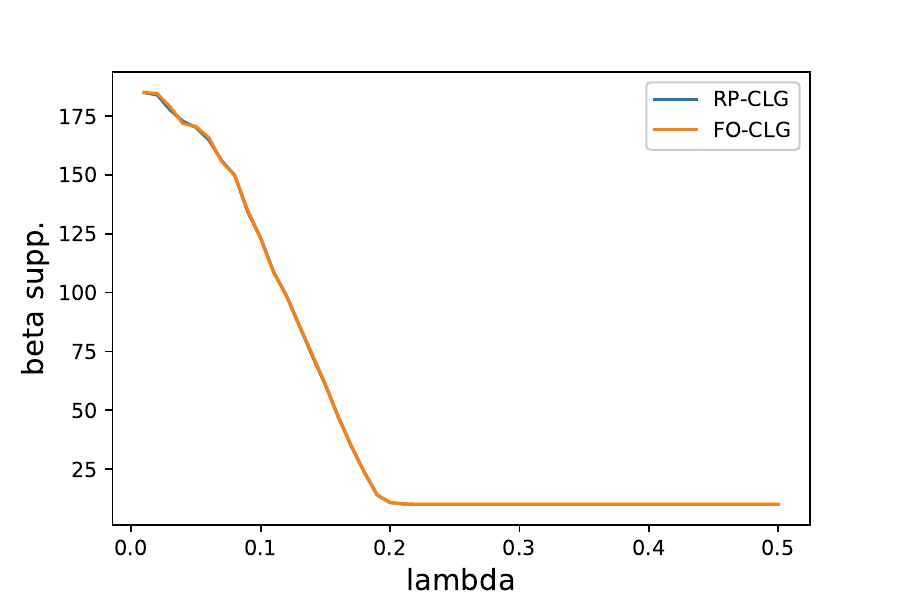} \\
		$\kappa$ (here, $\kappa=\lambda/\lambda_{\max}$) & $\kappa$ (here, $\kappa=\lambda/\lambda_{\max}$)
\end{tabular}
\caption{L1-SVM solutions $(n=1K, p=100K)$ on a regularization path. [Left] Runtime (s) versus $\kappa$, defined in~\eqref{defn-kappa-lambda} [Right] Support size of $\B\beta$ denoted by ``beta supp(ort)", i.e., number of nonzero SVM coefficients $\B\beta$ for different values of $\kappa$, as available from \texttt{FO-CLG} and \texttt{RP-CLG}. (Note that the two profiles for ``beta supp" are identical).}\label{fig:path solution}
\end{figure}

\subsubsection{Synthetic datasets for large $n$ and small $p$}\label{sec: simulations-nlarge}

When $n \gg p$, we illustrate the usefulness of our constraint generation procedure: 

\begin{enumerate}
	\item[(\bf i)] \texttt{FO-CNG}: This is our constraint generation (CNG) method when initialized with a subsampling based first order (FO) heuristic (cf Section~\ref{sec:scalability-methods}). 
\end{enumerate}
For \texttt{FO-CNG}, we use a reduced cost convergence threshold of $\ep = 0.01$ (we limit the number of constraints to be added to at most $400$).
We compare \texttt{FO-CNG} with several benchmarks: \texttt{SGD}, \texttt{SCS}, \texttt{FOM}, \texttt{Gurobi} as discussed in Section~\ref{sec: simulations-plarge}. We do not present the results for \texttt{PSM} as this was found to be much slower on these instances with large $n$ (sometimes \texttt{PSM} would not run on these instances). 

Table~\ref{table:simulated-n-large2} presents the results for $\lam = 0.001 \lam_{\max}$ and $\lam = 0.01 \lam_{\max}$. The sub-tables of Table~\ref{table:simulated-n-large2} consider $(n,p) = ( 10K, 100)$, $(10K, 300)$ and $(50K, 100)$, and we present results averaged over 5 replications. As constraint generation leverages sparsity in $\B\xi$, to get an idea of the sparsity of the problem we are dealing with, we present (i) the support size of the solution $\hat{\B\xi}$ (computed by \texttt{Gurobi}); and (ii) the number of constraints $|\cI|$ in the restricted problem, upon termination of \texttt{FO-CNG}.
From Table~\ref{table:simulated-n-large2}, it can be seen that 
\texttt{FO-CNG} outperforms other methods by a factor $ 4X \sim 30X $. In particular, \texttt{FO-CNG} has better performance when $\lam$ is small (recall, we are considering the setting where $n \gg p$). This is because a small value of $\lam$ imparts less shrinkage on the SVM coefficients $\B\beta$, hence
the support size of $\hat{\B{\xi}}$ is small (i.e., the number of falsely-classified samples is small). As a result, the constraint generation method speeds up overall runtime making \texttt{FO-CNG} computationally attractive. Both \texttt{FO-CNG} and \texttt{Gurobi} (solving the full LP) reach high accuracy solutions. The accuracy of solutions obtained by \texttt{FO-CNG} is considerably higher compared to \texttt{SGD}, \texttt{SCS} and \texttt{FOM}. In this example ($n \gg p$), we observe that \texttt{SGD} works well compared to the examples in Section~\ref{sec: simulations-plarge} where $p \gg n$.

Finally, we note that in the examples considered in Table~\ref{table:simulated-n-large2} when $\lam$ is very large and the support size of $\hat{\B{\xi}}$ is close to $n$,
the runtime of \texttt{FO-CNG} is likely going to increase.

\begin{table}[h!]
	\centering
\resizebox{\textwidth}{!}{
\begin{tabular}{lc cc cc cc c} 
\multicolumn{7}{c}{$\lam = 0.001 \lam_{\max}$} \\
\cmidrule{1-7}  & \multicolumn{2}{c}{$n=10K, p=100$} & \multicolumn{2}{c}{$n=10K, p=300$} & \multicolumn{2}{c}{$n=50K, p=100$}  \\ \hline
Method          & Time (s)  & ARA & Time (s) & ARA  & Time (s)  & ARA   \\ \hline
\texttt{SGD} & 54.2(4.05) & 1.8e-02(8.78e-04) & 117.4(2.89) & 4.5e-02(2.19e-03) & 313.1(3.19) & 7.4e-03(3.77e-04)  \\ 
\texttt{SCS} & 51.4(4.03) & 1.7e-04(4.31e-05) & 117.6(1.74) & 5.2e-04(1.90e-04) & 241.9(11.19) & 3.9e-05(7.35e-06)  \\ 
\texttt{FOM} & 170.1(3.40) & 2.8e-03(1.39e-04) & 207.0(9.03) & 5.9e-03(1.12e-04) & 1147.4(35.85) & 5.9e-04(1.73e-05)  \\ 
\texttt{Gurobi} & 66.3(3.52) & 0.0e+00(0.00e+00) & 133.8(5.48) & 0.0e+00(0.00e+00) & 626.6(42.49) & 3.4e-16(1.95e-16)  \\ 
\texttt{FO-CNG} & 3.0(0.06) & 1.3e-05(1.12e-05) & 6.3(0.24) & 2.1e-15(1.99e-16) & 20.8(0.22) & 0.0e+00(0.00e+00)  \\ \hline 
$\|\hat{\B{\xi}}\|_0$ \& $| \cI |$ & \multicolumn{2}{c}{87.8 ~\&~ 443.8}       & \multicolumn{2}{c}{88.8 ~\&~ 362.2}     & \multicolumn{2}{c}{538.6 ~\&~ 2277.8}      \\ \hline 
 && && && \\
\end{tabular} } \\
\resizebox{\textwidth}{!}{
\begin{tabular}{lc cc cc c} 
 \multicolumn{7}{c}{ $\lambda=0.01 \lam_{\max}$ } \\ \hline
& \multicolumn{2}{c}{$n=10K, p=100$} & \multicolumn{2}{c}{$n=10K, p=300$} & \multicolumn{2}{c}{$n=50K, p=100$}  \\ \hline
Method          & Time (s)  & ARA & Time (s) & ARA  & Time (s)  & ARA   \\ \hline
\texttt{SGD} & 55.6(2.88) & 4.0e-03(3.51e-04) & 130.1(4.50) & 8.2e-03(2.38e-04) & 326.7(5.36) & 3.6e-03(1.28e-04)  \\ 
\texttt{SCS} & 28.1(3.99) & 2.3e-05(3.62e-06) & 79.6(7.68) & 3.2e-05(6.57e-06) & 112.9(7.21) & 1.2e-05(3.27e-06)  \\ 
\texttt{FOM} & 47.8(0.90) & 5.8e-04(1.05e-05) & 68.8(4.00) & 1.0e-03(3.21e-05) & 297.4(0.82) & 1.3e-04(5.87e-06)  \\ 
\texttt{Gurobi} & 49.3(1.84) & 9.1e-17(5.92e-17) & 104.1(5.43) & 7.3e-17(6.56e-17) & 377.8(15.05) & 2.1e-16(1.89e-16)  \\ 
\texttt{FO-CNG} & 9.0(0.26) & 3.5e-06(3.16e-06) & 11.4(0.32) & 8.3e-06(3.14e-06) & 32.4(0.83) & 1.3e-06(4.80e-07)  \\ \hline 
$\|\hat{\B{\xi}}\|_0$ \& $| \cI |$ & \multicolumn{2}{c}{501.4 ~\&~ 1090.2}       & \multicolumn{2}{c}{472.2 ~\&~ 813.8}     & \multicolumn{2}{c}{2559.4 ~\&~ 3472.8}      \\ \hline 
\end{tabular}}	
	\caption{{\bf L1-SVM, Synthetic dataset ($n \gg p$)}: Training time for L1-SVM versus state-of-art methods on synthetic datasets. We show two different values of $\lambda$: $\lam = 0.001 \lam_{\max}$ [top table] and $\lambda = 0.01\lambda_{\max}$ [bottom table]. Each table presents results for different values of $(n,p)$ with $n \gg p$. For each table, the last row presents the number of nonzeros in $\B\xi$ in an optimal solution to the problem and the number of active constraints $|\cI|$ upon termination of \texttt{FO-CNG}. } \label{table:simulated-n-large2}
\end{table}

\subsubsection{Synthetic datasets for $n \approx p$}\label{sec: simulations-nplarge}
We study the performance of the algorithms when both $n$ and $p$ are comparable and moderately large. In this case, we use the following method with both column and constraint generation:
\begin{enumerate}
	\item[({\bf j})] \texttt{FO-CLCNG}: This is the combined column-and-constraint generation method, denoted by the shorthand (CLCNG) (i.e., Algorithm~4), initialized with the first order method discussed\footnote{For the subsampling heuristic, once the average
estimate was obtained, we took the top $200$ coefficients (in terms of absolute value) to
initialize the set of columns for column generation.} in Section~\ref{sec: FOM-CCG}. The column/constraint generation reduced cost thresholds are set to be equal $\ep:= \ep_1 = \ep_2$. 
\end{enumerate}
For \texttt{FO-CLCNG}, we use $\ep = 0.01$ and limit the maximum number of columns/constraints that are added at an iteration to 400. 
We compare \texttt{FO-CLCNG} with benchmarks including \texttt{SGD}, \texttt{SCS}, \texttt{FOM}, \texttt{Gurobi} under the same setting as  Section \ref{sec: simulations-plarge}. Once again, \texttt{PSM} is found to be significantly slow as $n$ is large, hence we do not include it in our results. 
Table~\ref{table:simulated-np-large2} presents the results for $\lam = 0.01 \lam_{\max}$ and $\lam = 0.1 \lam_{\max}$, with $(n,p) = ( 3K, 3K)$, $(2K, 5K)$ and $(5K, 2K)$. Note that in these examples, $\M{X}$ is dense so we do not consider larger problem-sizes---larger $n,p$ values with a sparse $\M{X}$ are considered in Section~\ref{sec:real-datasets-expts}. 

Table~\ref{table:simulated-np-large2} presents runtimes (s) and ARA values (means and standard errors across  $5$ independent experiments). In addition, to get an idea of the sparsity level of the problem, we also present the support sizes of the solutions $\hat{\B{\beta}}$ and 
$\hat{\B{\xi}}$ (computed by \texttt{Gurobi}). We also list the number of columns and constraints ($|\cJ| $ and $|\cI|$) in the resctricted problem, upon termination of \texttt{FO-CLCNG}. 
As shown in the sub-tables of Table~\ref{table:simulated-np-large2}, for $\lam = 0.01 \lam_{\max}$, \texttt{FO-CLCNG} has a $7X\sim 30X $   speedup over other methods; for $\lam = 0.1 \lam_{\max}$, \texttt{FO-CLCNG} has a $4X \sim 50X$  speedup over other methods. At the same time, it is important to note that the ARA of \texttt{FO-CLCNG} is around $10^{-5} \sim 10^{-6}$ --- this is notably 
better than that of \texttt{SCS}, \texttt{FOM} and \texttt{SGD}.

\begin{table}[h!]
	\centering
\resizebox{\textwidth}{!}{
\begin{tabular}{lc cc cc cc c} 
\multicolumn{7}{c}{$\lam = 0.01 \lam_{\max}$} \\
\cmidrule{1-7}  & \multicolumn{2}{c}{$n=3K, p=3K$} & \multicolumn{2}{c}{$n=2K, p=5K$} & \multicolumn{2}{c}{$n=5K, p=2K$}  \\ \hline
Method          & Time (s)  & ARA & Time (s) & ARA  & Time (s)  & ARA   \\ \hline
\texttt{SGD} & 393.3(32.06) & 1.4e-01(7.79e-03) & 387.3(25.89) & 4.2e-01(1.20e-02) & 486.9(46.32) & 4.5e-02(1.28e-03)  \\ 
\texttt{SCS} & 281.8(34.91) & 1.7e-04(1.73e-05) & 261.7(37.13) & 3.0e-04(7.05e-05) & 426.3(37.83) & 1.3e-04(2.44e-05)  \\ 
\texttt{FOM} & 100.6(9.74) & 5.2e-03(1.27e-04) & 93.4(5.33) & 7.9e-03(1.83e-04) & 142.3(12.46) & 3.0e-03(5.39e-05)  \\ 
\texttt{Gurobi} & 103.9(27.61) & 0.0e+00(0.00e+00) & 109.8(31.84) & 0.0e+00(0.00e+00) & 585.9(57.41) & 0.0e+00(0.00e+00)  \\ 
\texttt{FO-CLCNG} & 14.1(0.40) & 3.5e-07(2.59e-07) & 9.9(0.11) & 8.2e-06(2.11e-06) & 20.2(0.44) & 3.6e-06(2.82e-06)  \\ \hline 
$\|\hat{\B{\beta}}\|_0$ \& $| \cJ |$   & \multicolumn{2}{c}{170.2 ~\&~ 646.4}       & \multicolumn{2}{c}{162.0 ~\&~ 692.0}     & \multicolumn{2}{c}{188.0 ~\&~ 634.2}      \\ \hline 
$\|\hat{\B{\xi}}\|_0$ \& $| \cI |$ & \multicolumn{2}{c}{132.8 ~\&~ 565.2}       & \multicolumn{2}{c}{87.6 ~\&~ 533.4}     & \multicolumn{2}{c}{221.4 ~\&~ 668.8}      \\ \hline 
 && && && \\
\end{tabular} } \\
\resizebox{\textwidth}{!}{
\begin{tabular}{lc cc cc c} 
 \multicolumn{7}{c}{ $\lambda=0.1 \lam_{\max}$ } \\ \hline
& \multicolumn{2}{c}{$n=3K, p=3K$} & \multicolumn{2}{c}{$n=2K, p=5K$} & \multicolumn{2}{c}{$n=5K, p=2K$}  \\ \hline
Method          & Time (s)  & ARA & Time (s) & ARA  & Time (s)  & ARA   \\ \hline
\texttt{SGD} & 521.6(40.17) & 3.5e-03(1.80e-04) & 527.9(26.79) & 5.8e-03(6.02e-04) & 539.5(24.21) & 2.4e-03(2.41e-04)  \\ 
\texttt{SCS} & 150.5(22.98) & 3.4e-04(4.09e-05) & 178.7(9.48) & 3.8e-04(1.12e-04) & 259.2(26.38) & 8.2e-05(1.40e-05)  \\ 
\texttt{FOM} & 62.2(4.98) & 2.6e-04(1.36e-05) & 60.5(2.72) & 4.4e-04(4.86e-05) & 71.0(5.52) & 1.3e-04(1.47e-05)  \\ 
\texttt{Gurobi} & 47.6(6.74) & 0.0e+00(0.00e+00) & 44.4(5.67) & 0.0e+00(0.00e+00) & 354.9(116.66) & 0.0e+00(0.00e+00)  \\ 
\texttt{FO-CLCNG} & 11.7(0.14) & 3.1e-06(1.65e-06) & 8.3(0.13) & 8.2e-07(6.71e-07) & 18.7(0.16) & 1.2e-06(6.69e-07)  \\ \hline 
$\|\hat{\B{\beta}}\|_0$ \& $| \cJ |$   & \multicolumn{2}{c}{56.8 ~\&~ 255.8}       & \multicolumn{2}{c}{68.4 ~\&~ 322.2}     & \multicolumn{2}{c}{48.0 ~\&~ 235.8}      \\ \hline 
$\|\hat{\B{\xi}}\|_0$ \& $| \cI |$ & \multicolumn{2}{c}{750.8 ~\&~ 1053.2}       & \multicolumn{2}{c}{504.8 ~\&~ 751.2}     & \multicolumn{2}{c}{1287.0 ~\&~ 1657.4}      \\ \hline 
\end{tabular}}	
	\caption{{\bf L1-SVM, Synthetic dataset ($n \approx p$)}: Training time for L1-SVM versus state-of-art methods on synthetic datasets. We show two different values of $\lambda$: $\lam = 0.01 \lam_{\max}$ [top table] and $\lambda = 0.1\lambda_{\max}$ [bottom table]. Each table presents results for different values of $(n,p)$ with $n \approx p$. For each table, the last row presents the number of nonzeros in $\B\beta$ and $\B\xi$ in an optimal solution to the problem and the number of active variables and constraints ($|\cJ| $ and $|\cI|$) upon termination of \texttt{FO-CLCNG}. } \label{table:simulated-np-large2}
\end{table}

\medskip

\noindent {\textbf{Performance under different tolerance levels $\ep$:}}
Table~\ref{table:simulated-np-large2} presents results of \texttt{FO-CLCNG} for a fixed value of $\ep = 0.01$. In Figure \ref{fig:different tol2}, to understand sensitivity to the choice of $\epsilon$, we present the runtime and ARA of \texttt{FO-CLCNG} under different choices of $\ep \in \{0.01, 0.03, 0.1, 0.3, 1\}$ for $\lam = 0.01 \lam_{\max}$.
The presented results are the mean of $5$ independent replications. Similar to Figure~\ref{fig:different tol}, it can be seen in 
Figure~\ref{fig:different tol2}, that as $\ep$ decreases from $1$ to $0.01$, the ARA of \texttt{FO-CLCNG} decreases from $10^{-1}$ to $10^{-5} \sim 10^{-6}$, while the runtime only increases slightly. Therefore a tolerance of $\ep  = 0.01$ is sufficiently small and leads to a fairly high accuracy for the numerical experiments considered.

\begin{figure}[h!]
\centering
\begin{tabular}{cc}
		\includegraphics[width=0.48\textwidth,trim = .3cm .6cm .3cm .3cm,  clip = true]{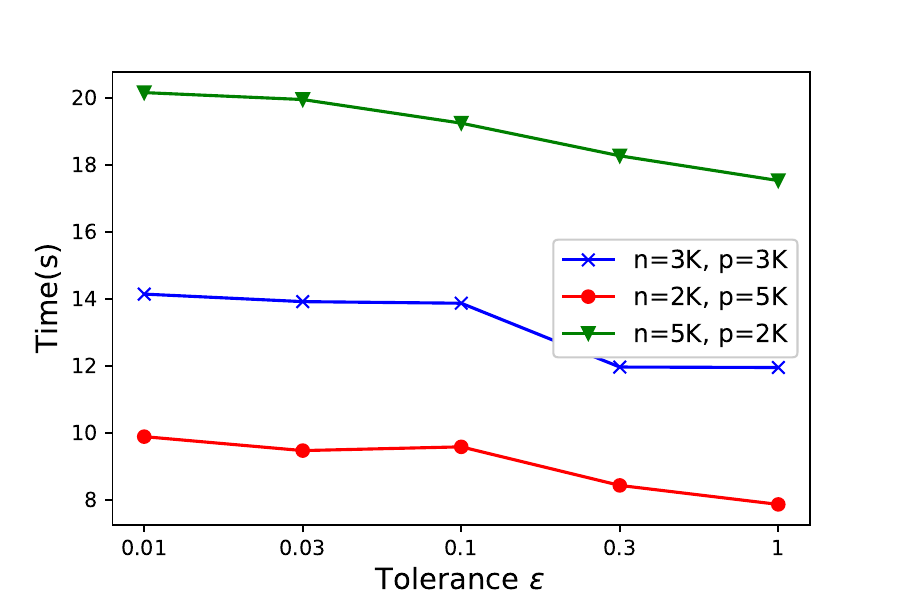}&
		\includegraphics[width=0.48\textwidth,trim = .2cm .6cm .3cm .3cm,  clip = true]{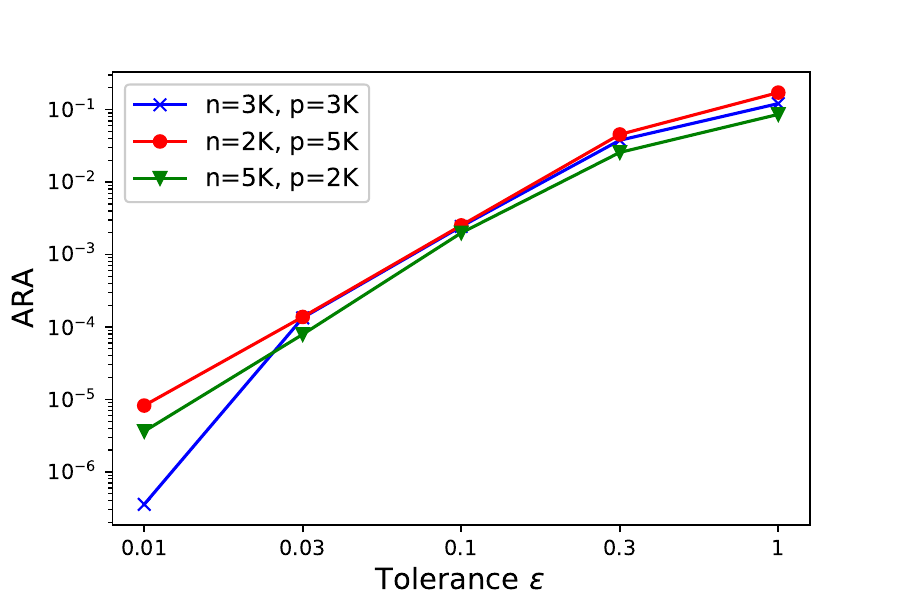} \\
		Tolerance ($\epsilon$) & 	Tolerance ($\epsilon$)
\end{tabular}

		\caption{
		Runtime (s) and ARA of \texttt{FO-CLCNG} under different choices of the tolerance threshold $\ep$. We consider 3 different $(n,p)$-values. (This figure mirrors Figure~\ref{fig:different tol} that shows results for column generation alone: \texttt{FO-CLG}).}
	\label{fig:different tol2}
\end{figure}

\subsubsection{Real datasets with large $n$ and $p$}\label{sec:real-datasets-expts}

Finally, we assess the quality of our hybrid column-and-constraint generation method  (\texttt{FO-CLCNG}) on large real datasets. For a fair comparison, we compare different methods in terms of their ability to optimize the same L1-SVM optimization problem. 
We consider three popular open-source datasets \texttt{rcv1}, \texttt{news20} 
and \texttt{real-sim}
that can be found at \url{https://www.csie.ntu.edu.tw/~cjlin/liblinear/}. We also consider a larger dataset derived from \texttt{rcv1}: We augment the original features of \texttt{rcv1} with noisy features obtained by randomly selecting a collection of features from the original dataset, and randomly permuting the rows of the selected features. We call this augmented dataset  
\texttt{rcv1-aug}.
Similarly, we form an augmented version of 
the dataset \texttt{real-sim} that we denote by
\texttt{real-sim-aug}.
Note that  all these datasets are sparse---\#nonzero entries in $\M{X}$, denoted by 
$\text{nnz}(\M{X})$, is quite small compared to $np$; and we use sparse matrices (scipy implementation) to deal with sparse matrix/vector multiplications.

We compare \texttt{FO-CLCNG} (we use a combination of column and constraint generation as $n,p$ are both large) 
with benchmark methods \texttt{Gurobi}, \texttt{SGD} and \texttt{SCS}. All methods are run under the settings explained in Section \ref{sec: simulations-plarge}, except that here we run \texttt{SGD} for $20,000$ epochs to arrive at an ARA$\sim 10^{-2}$--$10^{-3}$.

Table \ref{table: compare on datasets} presents the sizes of the datasets considered  and the runtime (s), ARA for different algorithms. We consider a sequence of 11 values of $\lambda = \kappa \lam_{\max}$ 
where $\kappa$ lies in the range $\kappa \in [0.003,0.3]$. To make the comparisons fair, all algorithms are run independently for different $\lambda$-values; and we present the average runtime and ARA for a value of $\lambda$. In the last row of each sub-table, we present the minimum, mean and maximum of the value $\|\hat{\B{\beta}}\|_0 + \|\hat{\B{\xi}}\|_0 $ over the path of $\lam$-values. The numbers are presented in the form of the triplet ``(min, mean, max)"; and 
$(\hat{\B{\beta}},\hat{\B{\xi}})$ is the solution obtained from
\texttt{FO-CLCNG}.

On the datasets \texttt{rcv1} and \texttt{rcv1-aug}, our proposed method \texttt{FO-CLCNG} outperforms \texttt{SCS} and \texttt{Gurobi} in runtime by a factor of $3X\sim 8X$ and delivers solutions of higher accuracy (ARA).
For other two datasets \texttt{news20} and \texttt{real-sim-aug}, \texttt{Gurobi} and \texttt{SCS} would run out of memory for some values of $\lambda$ (we set a 16GB memory limit). 
As our column-and-constraint generation procedure operates on a smaller reduced problem, it consumes less memory. The solutions of \texttt{FO-CLCNG} have high accuracy with ARA around $10^{-6}$. By comparing \texttt{FO-CLCNG} and \texttt{SGD}, we see that the runtime of 20,000 epochs of \texttt{SGD} is comparable to that of \texttt{FO-CLCNG} on the \texttt{rcv1} dataset, while it is slower than \texttt{FO-CLCNG} on other datasets. Note however that the optimization accuracy of the SGD solutions is significantly worse compared to \texttt{FO-CLCNG}.

\begin{table}[h!]
\begin{minipage}[b]{0.49\linewidth}\centering
\begin{tabular}{lcc}
\hline
\multicolumn{3}{c}{Dataset: rcv1}                         \\ 

\multicolumn{3}{c}{(n=16194, p=47237, nnz=1.20e+06)}           \\ 
\hline 

method   & Time (s) & ARA  \\ \hline  
\texttt{Gurobi}   &   442.0   & 0.0e+00      \\ 
\texttt{SCS}      &   1185.9   & 2.2e-04      \\ 
\texttt{SGD}      &   256.3   & 1.2e-02      \\ 
\texttt{FO-CLCNG} &   142.9   & 1.0e-06      \\  \hline
$\|\hat{\B{\xi}}\|_0+\|\hat{\B{\beta}}\|_0$ & \multicolumn{2}{c}{ (3212, 6538, 10580) } \\ 
\hline
\end{tabular}
\end{minipage}
\hfill
\begin{minipage}[b]{0.49\linewidth}\centering
\begin{tabular}{lcc}
\hline
\multicolumn{3}{c}{Dataset: rcv1-aug}                         \\ 

\multicolumn{3}{c}{(n=16194, p=236185, nnz=6.00e+06)}           \\ 
\hline 

method   & Time (s) & ARA  \\ \hline  
\texttt{Gurobi}   &   1590.5   & 0.0e+00      \\ 
\texttt{SCS}      &   3318.8   & 9.2e-04      \\ 
\texttt{SGD}      &   1419.3   & 1.4e-02      \\ 
\texttt{FO-CLCNG} &   442.2   & 1.1e-06      \\  \hline
$\|\hat{\B{\xi}}\|_0+\|\hat{\B{\beta}}\|_0$ & \multicolumn{2}{c}{ (3997, 6765, 10512) } \\ 
\hline
\end{tabular}
\end{minipage}

\vspace{0.2cm}
\begin{minipage}[b]{0.49\linewidth}\centering
\begin{tabular}{lcc}
\hline
\multicolumn{3}{c}{Dataset: news20}                         \\ 

\multicolumn{3}{c}{(n=15997, p=1355191, nnz=7.31e+06)}           \\ 
\hline 

method   & Time (s) & ARA  \\ \hline  
\texttt{Gurobi}   &   -   & -      \\ 
\texttt{SCS}      &   -   & -      \\ 
\texttt{SGD}      &   2014.4   & 4.3e-03      \\ 
\texttt{FO-CLCNG} &   112.0   & 0.0e+00      \\  \hline
$\|\hat{\B{\xi}}\|_0+\|\hat{\B{\beta}}\|_0$ & \multicolumn{2}{c}{ (7106, 10010, 13853) } \\ 
\hline
\end{tabular}
\end{minipage}
\hfill
\begin{minipage}[b]{0.49\linewidth}\centering
\begin{tabular}{lcc}
\hline
\multicolumn{3}{c}{Dataset: real-sim-aug}                         \\ 

\multicolumn{3}{c}{(n=57847, p=104795, nnz=1.47e+07)}           \\ 
\hline 

method   & Time (s) & ARA  \\ \hline  
\texttt{Gurobi}   &   -   & -      \\ 
\texttt{SCS}      &   -   & -      \\ 
\texttt{SGD}      &   2908.9   & 2.5e-03      \\ 
\texttt{FO-CLCNG} &   955.3   & 0.0e+00      \\  \hline
$\|\hat{\B{\xi}}\|_0+\|\hat{\B{\beta}}\|_0$ & \multicolumn{2}{c}{ (11975, 18931, 26632) } \\ 
\hline
\end{tabular}
\end{minipage}

\caption{ L1-SVM on real datasets with both $n,p$ large. We compare our method \texttt{F0-CLCNG} versus other benchmarks (in terms of runtime and ARA) on a range of $\lambda$-values, as discussed in the text. A ``-'' means that the method would not run due to memory limitations and/or numerical problems. The last row of every sub-table provides the (minimum, average, maximum)-tuple of the support-size $\|\hat{\B{\xi}}\|_0+\|\hat{\B{\beta}}\|_0$ where the minimum, average and maximum values are taken across the sequence of $\lambda$-values considered.}
	\label{table: compare on datasets}

\end{table}

\subsection{Computational results for Group-SVM}\label{sec-results-group-SVM}
We now study the performance of the column generation algorithm presented in Section \ref{sec: group-CG} for the Group-SVM Problem \eqref{group-SVM-intro}.

\noindent {\bf{Data Generation:}} Here, the covariates are drawn from a multivariate Gaussian with covariance $\B\Sigma$. 
The $p$ covariates are divided into $G$ groups each of the same size $p_G$. Within each group, covariates have pairwise correlation of $\rho$, and covariates are uncorrelated across groups (all variances are equal). Half of the samples are from the $+1$ class with (population) mean $\B\mu_+ = \left( \mathbf{1}_{p_G}, \ldots, \mathbf{1}_{p_G},  \mathbf{0}_{p_G}, \ldots, \mathbf{0}_{p_G}\right)$ (the number of sub-vectors $\mathbf{1}_{p_G}$ is $k_0$); the remaining samples
from class $-1$ have (population) mean $\B\mu_- = -\B\mu_+$. In the following example we take $\rho = 0.1$, $p_G = 10$ and $k_0 = 10$.

\begin{table*}[h!]
	\centering
	\resizebox{\textwidth}{!}{
		\begin{tabular}{lcccccccc} 
\multicolumn{7}{c}{Group-SVM ($p \gg n$)}\\
			\cmidrule{1-7}  & \multicolumn{2}{c}{$n=100, p=10K$} & \multicolumn{2}{c}{$n=300, p=10K$} & \multicolumn{2}{c}{$n=100, p=30K$}  \\ \hline
			Method          & Time (s)  & ARA & Time (s) & ARA  & Time (s)  & ARA   \\ \hline
			\texttt{SCS} & 78.1(16.05) & 8.7e-04(1.10e-04) & 82.6(14.65) & 7.0e-04(2.89e-04) & 287.4(21.38) & 1.1e-02(5.55e-03)  \\ 
			\texttt{Gurobi} & 120.8(5.16) & 3.4e-17(5.91e-17) & 321.7(14.19) & 3.4e-17(6.85e-17) & 503.8(26.53) & 6.8e-17(6.80e-17)  \\ 
			\texttt{FO-CLG} & 2.1(0.17) & 1.0e-16(1.14e-16) & 3.8(0.13) & 1.8e-16(1.91e-16) & 2.3(0.14) & 6.7e-17(1.15e-16)  \\ \hline 
		\end{tabular}
	}
	\caption{Training time (s) and ARA for Group-SVM versus various benchmarks on synthetic datasets, $\lam = 0.1 \lam_{\max}$.}
	\label{table: compare group SVM}
\end{table*}

\noindent {\bf{Comparison with benchmarks:}} We compare our column generation method (\texttt{FO-CLG}) with \texttt{SCS} and \texttt{Gurobi}. Both \texttt{SCS} and \texttt{Gurobi} are called through CVXPY under its default settings. Our method \texttt{FO-CLG} (cf Section~\ref{sec: group-CG}) applies column generation after initialization with the block coordinate descent procedure (cf Section \ref{sec: AGD}). 
We use a smoothing parameter $\tau= 0.2$ (for the hinge-loss) and use a CD method restricted to the top $n$ groups obtained via correlation screening (cf Section~\ref{corr-screen-1}).
We  use a reduced cost tolerance of $\epsilon=0.01$ in the column generation method.
Table \ref{table: compare group SVM} shows the results on synthetic datasets with $(n,p) = (100, 10K)$, $(300, 10K)$ and $(100, 30K)$ and set $\lam = 0.1 \lam_{\max}$. The reported values are based on 5 independent replications. We can see that on these examples, \texttt{FO-CLG} outperforms the other two methods by a large factor $\ge 30X$ in runtime, and also delivers a solution with high accuracy (as seen from the ARA values).

\subsection{Computational results for Slope-SVM}\label{sec-results-slope-SVM}
We present the computational performance of the column-and-constraint generation methods presented in Section  \ref{sec:slope-CCG} for the Slope-SVM Problem \eqref{Slope-intro}. 
The synthetic data is 
simulated as in Section~\ref{sec: simulations-plarge} --- we set $n=100, \ k_0=10, \ \rho=0.1$. 

\noindent{\bf{Comparison when $\lambda_{i}$s are not all distinct:}} We are not aware of any publicly available specialized implementation for the Slope SVM problem.
We use the {\texttt{CVXPY}} modeling framework  to model (See Section~\ref{epi-slope-1}) the Slope SVM problem and solve it using state-of-the art solvers like   Ecos and {Gurobi}.
We first consider a special instance of the Slope penalty~\eqref{slope-norm} 
that corresponds to the coefficients $\lambda_{i} = 2\tilde{\lambda}$ for $i \leq k_0$ and $\lambda_{i} = \tilde{\lambda}$ for $i>k_0$;
where $\tilde{\lambda} = 0.01 \lambda_{\max}$. We solve the resulting problem with both the Ecos and {Gurobi} solvers, denoted by 
``\texttt{CVXPY Ecos}"  and  ``\texttt{CVXPY Gurobi}" (respectively).  We compare them with our proposed column-and-constraint generation algorithm, referred to as 
``{\texttt{FO-CLCNG}}".  
For our method, we first run the first order algorithm presented in Section \ref{sec: AGD} (for $\tau= 0.2$) restricted to the $10n$ columns with highest absolute correlations with the response (the remaining coefficients are all set to zero). 
The column-and-constraint generation algorithm (cf Section \ref{sec:slope-CCG}) uses a tolerance level of $\epsilon=0.001$. We limit the number of columns added at each iteration to 10. For reference, we also report the run time of our algorithm by excluding the time taken by  
the initialization step---this is denoted by ``\texttt{CLCNG wo FO}".  The results, averaged over 5 replications,  are presented in Table~\ref{table:Slope-CVX}. Results in Table~\ref{table:Slope-CVX} indicate that our proposed method {\texttt{FO-CLCNG}} exhibits a $50$X--$110$X improvement compared to competing solvers. In some cases, when $p \geq 50K$,  \texttt{CVXPY Ecos} encounters numerical problems and hence would not run. 
On the other hand, we note that {\texttt{FO-CLCNG}} has the best ARA for $p \ge 20K$ even when compared with \texttt{CVXPY Gurobi}. This may be because the model size is large and \texttt{CVXPY Gurobi} appears to use a low-accuracy termination condition.

\begin{table*}[h!]
	\centering
	{Slope-SVM ($p \gg n$)}
	
	\begin{tabular}{cccccccc}
\cmidrule{1-8} & \multicolumn{2}{c}{\texttt{FO-CLCNG}} & \multicolumn{1}{c}{\texttt{CLCNG wo FO}} & \multicolumn{2}{c}{\texttt{CVXPY Ecos}}      & \multicolumn{2}{c}{\texttt{CVXPY Gurobi}} \\ \hline
$p$            & Time (s)     & ARA              & Time (s)                         & Time (s)           & ARA                & Time (s)           & ARA         \\ \hline
10k            & 1.7(0.07)   & 1.3e-06             & 1.2(0.05)                       & 64.3(6.81)          & 7.0e-12              & 84.0(2.86)          & 0.0e+00             \\ 
20k            & 3.0(0.01)   & 0.0e+00             & 2.5(0.00)                       & 130.5(3.17)          & 1.7e-05                & 221.9(5.64)         & 1.7e-05             \\ 
50k            & 7.1(0.35)   & 0.0e+00             & 6.6(0.34)                       & -                   & -                   & 842.3(20.06)         & 9.1e-06             \\ 
100k           & 16.5(0.02)   & 0.0e+00             & 15.9(0.02)                       & -                   & -                   & 1837.3(70.44)         & 4.1e-06             \\ \hline
\end{tabular}
	\caption{Training times and ARA of our column-and-constraint generation method for Slope-SVM versus \texttt{CVXPY}---we took $\lambda_{i}/\lambda_{j} = 2$ for all $i \in [k_0]$ and $j > k_0$ (as described in the text). When the number of features are 
		in the order of tens of thousands (here, $n=100$), our proposed method ({\texttt{FO-CLCNG}}) enjoys nearly a 100X speedup in run time. A `--' symbol denotes that the corresponding algorithm encountered numerical problems.}
	\label{table:Slope-CVX}
\end{table*}

\noindent{\bf{Comparison when $\lambda_{i}$s are distinct:}} Here we consider a different sequence of $\lambda$-values: Following~\cite{bellec2018slope}, we set $\lambda_j = \sqrt{\log(2p/j)} \tilde{\lambda}$ with $\tilde{\lambda} = 0.01 \lambda_{\max}$. 
We observed that \texttt{CVXPY} could not handle
even small instances of this problem (as the $\lambda_j$s are distinct) --- in particular, the Ecos solver crashed for $n=100,p=200$. We compare our proposed method \texttt{FO-CLCNG}, 
with the first order method (cf. Section~\ref{sec: AGD}) applied to the full smoothed Slope-SVM problem with $\tau =0.2$.
Due to the high per iteration cost of the first order method (\texttt{FOM}), we terminate the method after a few iterations (the associated ARA is reported within parenthesis).
Table \ref{table:Slope-FOM} compares our methods---we use the same synthetic dataset as in the previous Slope-SVM example and average the results over 5 replications (the first order method was run for one replication due to its long run time).

\begin{table*}[h!]
	\centering
	{Slope-SVM ($p \gg n$)}
	
	\begin{tabular}{cccccc}
		\cmidrule{1-6} 
		&\multicolumn{2}{c}{\texttt{FO-CLCNG} } &\multicolumn{1}{c}{ \texttt{CLCNG wo FO} }  &\multicolumn{2}{c}{\texttt{FOM}} \\   
		\toprule
		$p$ & Time (s) & ARA  & Time (s) & Time (s) & ARA  \\ 
		\midrule
		10k & 1.4(0.11)  &  0.0e+00(0.00e+00)   & 0.5(0.05)    & 164.6    & 3.0e-01 \\ 
		20k & 1.6(0.09)  &  0.0e+00(0.00e+00)   & 0.7(0.04)    & 427.3    & 3.2e-01 \\ 
		50k & 3.3(0.20)  &  0.0e+00(0.00e+00)   & 2.4(0.14)    & 1633.8    & 3.2e-01 \\ 
		\bottomrule
	\end{tabular}
	\caption{Training times and ARA of our column-and-constraint generation for Slope-SVM with coefficients $\lambda_j = \sqrt{ \log(2p/j) } \tilde{\lambda}$ (as in the text) on synthetic datasets with large number of features (here, $n=100$). Our proposed approach (\texttt{FO-CLCNG}) outperforms the first order methods in runtimes (by at least 100-times). In addition, \texttt{FO-CLCNG} delivers solutions of higher accuracy compared to \texttt{FOM}.}\label{table:Slope-FOM}
\end{table*}

\section{Acknowledgements} The authors would like to thank the Action Editor and three anonymous reviewers for their helpful comments and constructive feedback that helped improve the paper. This research was supported in part, by grants from the Office of Naval Research: ONR-N000141812298 (YIP), the National Science Foundation: NSF-IIS-1718258 and IBM.

\appendix

\section{Appendix}\label{sec:appendix}

\subsection{Another formulation for the Slope norm}\label{epi-slope-1}
Without loss of generality, we consider a vector $\B\alpha \geq \M{0}$. 
 Now note that for every $m \in [p]$, we can represent 
 $\alpha_{(1)} + \ldots + \alpha_{(m)} \leq s_{m}$ as 
\begin{equation}\label{slope-rob-1}
\alpha_{(1)} + \ldots + \alpha_{(m)} \leq s_{m}, ~~ \B\alpha \geq \M{0} \iff  \begin{cases} \M{0} \leq \B\alpha \leq \theta_{m} \M{1} + \M{v}_{m} \\ m \theta_{m} + \M{1}^T\M{v}_{m} \leq s_{m} \end{cases} 
\end{equation}
with variables $\B\alpha, \M{v}_{m} \in \mathbb{R}^{p}, \theta_{m} \in \mathbb{R}$ and $\M{1} \in \mathbb{R}^p$ being a vector of all ones. 
Note that the rhs formulation in~\eqref{slope-rob-1} has $O(p)$ variables and $O(p)$-constraints. 
Note that we can write:
 $$\sum_{j=1}^p \lambda_j  \alpha_{(j)}  = \sum_{m=1}^{p} \tilde{\lambda}_{m} (\alpha_{(1)} + \ldots + \alpha_{(m)})  $$
 where, $\tilde{\lambda}_{m} = \lambda_{m} - \lambda_{m-1}$ for all $m \in \{ 1, \ldots, p \}.$ 
Therefore, representing $\sum_{j=1}^p \lambda_j  \alpha_{(j)} \leq \eta$ will require a representation~\eqref{slope-rob-1} for $m = 1, \ldots, p$ --- this will lead to a formulation with $O(p^2)$ variables and $O(p^2)$ constraints, which can be quite large as soon as $p$ becomes a few hundred.

\bibliographystyle{abbrvnat}
\small{\bibliography{L1-SVM_2021_0325.bib}}
			
\end{document}